\documentclass[journal]{IEEEtran}
\usepackage{amsmath,amsfonts}
\usepackage{algorithm}
\usepackage{array}
\usepackage[caption=false,font=normalsize,labelfont=sf,textfont=sf]{subfig}
\usepackage{textcomp}
\usepackage{stfloats}
\usepackage{url}
\usepackage{verbatim}
\usepackage{graphicx}
\usepackage{cite}
\usepackage{algorithm}
\usepackage{adjustbox}
\usepackage{algpseudocode}
\usepackage{booktabs}
\usepackage{medmath}
\usepackage{amssymb}
\usepackage{pgfplots}
\usepackage{pgfplotstable}
\pgfplotsset{compat=1.18}
\usepgfplotslibrary{fillbetween}
\usepackage{amsthm}
\usepackage{titlesec}
\usepackage[final]{microtype}
\microtypesetup{
  protrusion=true,
  expansion=false   
}

\usepackage{amsmath,amssymb}
\allowdisplaybreaks
\usepackage{enumitem}
\setlist{nosep,leftmargin=*}

\usepackage{graphicx}
\usepackage{booktabs}
\usepackage[hang,flushmargin]{footmisc}

\usepackage[hidelinks]{hyperref}

\titlespacing*{\section}
{0pt}{0.6\baselineskip}{0.4\baselineskip}

\titlespacing*{\subsection}
{0pt}{0.5\baselineskip}{0.3\baselineskip}

\titlespacing*{\subsubsection}
{0pt}{0.4\baselineskip}{0.2\baselineskip}
\titlespacing*{\paragraph}
{0pt}{0.25\baselineskip}{0.5em}

\usepackage{tikz}
\pgfplotsset{compat=1.18}
\usetikzlibrary{shapes.geometric, arrows.meta, positioning, fit}
\pgfplotsset{
    hxaiRankPlot/.style={
        width=0.48\textwidth,
        height=0.38\textwidth,
        xlabel={$\epsilon$ (log scale)},
        ylabel={Rank (lower is better)},
        xmode=log,
        log basis x={10},
        ymin=0,
        ymax=180,
        y dir=reverse,
        grid=both,
        major grid style={gray!30},
        minor grid style={gray!15},
        tick align=outside,
        tick style={black},
        legend style={
            font=\scriptsize,
            draw=none,
            fill=white,
            at={(0.98,0.02)},
            anchor=south east
        },
        legend cell align=left,
        mark size=2.2pt,
        line width=0.9pt,
    }
}

\usepackage{enumitem}
\newtheorem{theorem}{Theorem}[section]

\newtheorem{definition}{Definition}

\usepackage{booktabs}
   \usepackage{medmath} 
\usepackage{tikz}
\usetikzlibrary{arrows.meta, positioning, shapes.geometric, calc}
\algnewcommand\Input{\item[\textbf{Input:}]}
\algnewcommand\Output{\item[\textbf{Output:}]}
\usepackage{url}
\begin{document}

\title{HXAI: Hierarchical Privacy-Preserving Explainable AI in Distributed Energy Systems}

\author{
Poushali Sengupta,
Sabita Maharjan,
Frank Eliassen,
and Yan Zhang%
\thanks{
Poushali Sengupta, Sabita Maharjan, Frank Eliassen, and Yan Zhang are with the Department of Informatics, University of Oslo, 0316 Oslo, Norway.
}
}
\markboth{Journal of \LaTeX\ Class Files,~Vol.~14, No.~8, August~2021}%
{Shell \MakeLowercase{\textit{et al.}}: A Sample Article Using IEEEtran.cls for IEEE Journals}


\maketitle

\begin{abstract}
Balancing electricity demand and supply is increasingly difficult due to the inherent intermittency of renewable power generation and the stochastic power consumption. Grid operators require fine-grained, decision-relevant insights into household energy consumption to manage peak loads and design responsive tariffs, but increased transparency at this level raises significant privacy concerns. Traditional methods for explainable AI (XAI) can reveal sensitive information, while standard privacy techniques often reduce the usefulness of explanations. To address this issue, we introduce HXAI, a hierarchical framework that preserves privacy while enabling reasonable explainable analysis for grid-level demand management. HXAI consists of two main components:
(1) a local model that generates fine-grained explanations within a secure, private environment, and
(2) a zonal model that aggregates these explanations to support grid-level analysis while enforcing privacy through flexible privacy-budget management.
We explicitly limit cumulative privacy exposure under repeated operator queries and show that the proposed framework preserves decision-relevant information without compromising household privacy. Experiments on both simulated and real-world energy datasets demonstrate that HXAI provides useful insights for zonal load management while ensuring that appliance-level consumption remains local and is never transmitted to grid operators. Our results show that preserving the semantic structure of explanations, rather than minimizing numerical error, is the key to XAI under differential privacy.
This framework provides a way to achieve both privacy and explainability in energy management.
\end{abstract}


\section{Introduction}

Electricity grids worldwide are transforming due to the rapid integration of
renewable energy, increased household electrification, and widespread use of
IoT devices~\cite{faghosh2019iotenergy}. This evolution has increased volatility
in both supply and demand due to the inherent intermittency associated with
renewable energy sources and increasingly dynamic patterns of household
electricity consumption.\footnote{Renewable generation introduces natural
fluctuations~\cite{morstyn2018review}, while household demand variability has
intensified due to electric heating and cooling, electric vehicles, and other
high-power appliances~\cite{li2021evload}.} As a result, timely insights are
critical for understanding demand spikes, since even minor changes in local
consumption can lead to voltage instability and increased operational
challenges~\cite{kundur1994power}.\footnote{For example, localized demand surges
can trigger voltage drops on low-voltage feeders, increase transformer loading,
and necessitate conservative reserve margins to maintain grid stability.} While
individual households may seem inconsequential, collectively, their behaviors
can significantly influence the load curve, peak demand timing, and stress on
distribution networks~\cite{palensky2011demand}. Notably, high-consumption
devices such as heaters and appliances contribute to peak demand but are often
hidden from grid operators due to privacy concerns of household consumption
data~\cite{efthymiou2010privacy}. This lack of visibility hinders operators’
ability to anticipate surges, implement effective demand-response programs, and
design stability-focused tariffs~\cite{albadi2007demand}.

\par
Explainable artificial intelligence (XAI) offers a promising way to improve
understanding of energy consumption, using techniques such as SHAP and LIME to
identify patterns and the contributions of appliances to demand
spikes~\cite{lundberg2017shap,ribeiro2016lime}. However, applying XAI at the
household level is challenging due to privacy concerns, as detailed
explanations can reveal sensitive data~\cite{molnar2019explainable}. Simply
adding noise for privacy, such as Laplace noise to SHAP values, can undermine
the clarity and reliability needed for decision-making~\cite{abadi2016dpdl}.
Current privacy-preserving methods, such as federated learning, often provide
only broad statistics without in-depth, appliance-level insights and struggle
with effective privacy management and complex
queries~\cite{mcmahan2017fl}.\footnote{These approaches typically lack explicit
mechanisms for controlling cumulative privacy leakage or supporting fine-grained
explanation queries.} This leads to three key requirements:
(i) fine-grained, decision-relevant explainable signals that do not expose
household power consumption behavior, (ii) explainable zonal summaries that
remain useful under privacy-preserving noise, and (iii) explicit control of
cumulative privacy leakage under repeated analytical queries.

\noindent To tackle these challenges, we propose HXAI, a hierarchical framework
for privacy-preserving explainability in distributed energy systems,
supporting explainable analysis for grid-level demand management. HXAI splits
responsibilities into two layers: at the local layer, each household generates
SHAP explanations of its device-level consumption within a secure
private-databox, ensuring no raw data leakage~\cite{costan2016sgx}. At the zonal
layer, an aggregator receives differentially private summaries based on a
dynamically negotiated privacy budget between households and the energy
provider~\cite{dwork2014dp}. This dynamic negotiation allows households to
control their data exposure while providing operators with meaningful
aggregated explanations.

Existing privacy-preserving explainability approaches primarily aim to minimize
numerical distortion of explanations under differential privacy. In contrast,
HXAI is built around a different principle: \emph{semantic stability of
explanations}. We argue that, under privacy constraints, preserving the
\emph{structural invariants} of explanations, such as feature identity, sign,
dominance, and rank ordering, is more critical for explainability than
preserving exact numerical values~\cite{lipton2018mythos}. HXAI operationalizes
this principle through hierarchical aggregation, ensuring that appliance-level
explanations remain noise-free within secure local environments, while only
differentially private, aggregated explanation summaries are shared at the
zonal level.\footnote{Unlike federated learning, which focuses on collaborative
model training via the aggregation of gradients or parameters while keeping raw
data local~\cite{mcmahan2017fl}, HXAI does not perform distributed model
optimization but instead governs how explanations are aggregated and disclosed
across system layers.} HXAI fundamentally differs in both objective and
abstraction: it does not aim to learn a shared global model. Instead, HXAI
treats explanations as first-class analytical objects. Local models remain
household-specific, and only differentially private explanation summaries, not
parameters, gradients, or predictions, are aggregated. As a result, HXAI
enables appliance-aware explainability without requiring model alignment,
shared representations, or centralized optimization, which are core
assumptions in federated learning.

HXAI also integrates a mechanism for multiple-query privacy accounting,
limiting cumulative leakage when operators issue successive
queries~\cite{dwork2010composition}. Once the privacy threshold is met,
households can automatically deny further queries, preventing overexposure.
Theoretical analysis confirms that the aggregated SHAP explanations remain
unbiased under differential privacy noise~\cite{dwork2014dp}. Experimental
results show that HXAI maintains explainability and utility-preserving
insights, even under strong privacy constraints, allowing grid operators to
monitor behavioral trends and manage peak loads while safeguarding household
privacy.In summary, HXAI provides a principled and privacy-safe foundation for distributed energy intelligence, addressing the trade-off between household privacy and decision-relevant demand analysis for grid operations. This work makes the following key contributions toward privacy-safe and explainable grid-load balancing in distributed energy systems:
\begin{enumerate}[leftmargin=*, itemsep=0pt, topsep=0pt]
    \item \textbf{Semantic-Stability–Preserving Explainability under Differential Privacy.} 
    We propose HXAI, a hierarchical explainability framework that provides fine-grained, decision-relevant explanations while preserving semantic structure (feature identity, ranking, and directional influence) under differential privacy.
    \item \textbf{Adaptive, Query-Aware Privacy Control.}  
    HXAI introduces an \textbf{Information Control Agreement (ICA)} that governs permitted queries, output granularity, and cumulative privacy budgets to control long-term privacy leakage under repeated analysis.
    \item \textbf{Utility-Preserving Zonal Visibility with Guarantees.}  
    We provide theoretical and empirical evidence that HXAI maintains stable, useful zonal explanations even when individual contributions are protected by privacy-preserving noise.
\end{enumerate}
\noindent Collectively, these contributions establish HXAI as a practical and theoretically grounded foundation for privacy-preserving, explainable, and operationally relevant demand analysis for distributed smart-home environments.

\section{Related Work}

Since this work lies at the intersection of explainable AI, privacy, and smart
energy systems~\cite{fang2012smartgrid,palensky2011demand}, we organize the
related literature into three themes: (i) explainable AI for energy systems,
(ii) privacy-preserving machine learning, and (iii) differential privacy and
hierarchical architectures for distributed analytics. We discuss the
limitations of existing approaches and show how they motivate the hierarchical,
privacy-aware design of HXAI.

Explainable artificial intelligence (XAI) has been applied in energy
forecasting, home automation, and demand-response
analytics~\cite{ahmad2018energyxai}. Techniques such as SHAP, LIME, and various
feature-attribution methods offer transparency into the temporal or
appliance-level signals that influence energy
consumption~\cite{lundberg2017shap, ribeiro2016lime}. Previous studies have
utilized these techniques on smart-meter data or building-level loads to
enhance explainability~\cite{zhang2019energyxai}. However, existing XAI
implementations face two main structural limitations:
(i) explanations are typically generated centrally, necessitating access to
raw household data; and (ii) the resulting explanations are directly influenced
by fine-grained behavioural patterns, potentially revealing occupancy and
appliance-usage routines~\cite{efthymiou2010privacy}.To the best of our
knowledge, there is no prior work that provides fine-grained, device-level
explainability for high-impact residential loads while simultaneously ensuring
household privacy. Moreover, existing XAI approaches in smart energy systems are
largely disconnected from operational demand-management needs, where
understanding the contribution of major household loads (e.g., heating,
cooling, and electric vehicles) is essential for explaining peak demand, yet
cannot be directly exposed due to privacy constraints.

Privacy-preserving machine learning has advanced notably through cryptographic
and system-level techniques such as secure multi-party computation (SMPC),
homomorphic encryption (HE), and trusted execution
environments~\cite{goldreich1998secure, gentry2009fully, costan2016sgx}. While
HE and SMPC provide strong guarantees for protecting raw data during distributed
computation, their primary objective is data confidentiality rather than
explainability. Applying HE or SMPC to explanation mechanisms such as SHAP would
require encrypted evaluation of high-dimensional attribution functions across
many coalitions, leading to substantial computational and communication
overhead~\cite{bost2015machine}. This renders such approaches impractical for
real-time or repeated explainability queries in smart-grid settings.\footnote{More
importantly, cryptographic methods do not address explanation-specific
challenges, such as preserving feature identity, rank ordering, and directional
influence under repeated queries. While intermediate values are hidden, these
methods provide no mechanism for bounding cumulative information leakage or for
negotiating privacy--utility trade-offs over time~\cite{dwork2014dp}.} HXAI
therefore adopts a complementary approach: rather than encrypting all
computation, it treats explanations as first-class analytical objects,
generated locally within each household, and enforces privacy at the
explanation interface through differential privacy and query-aware accounting
before any aggregated information is shared.

Federated learning approaches ensure that raw data remains on-site while models
are collaboratively trained~\cite{mcmahan2017fl}. However, they lack mechanisms
for sharing appliance-level explainable signals with grid operators, and
existing explanation methods for federated models remain coarse-grained and
largely disconnected from household-specific or device-level operational
insights~\cite{li2020fedxai}. Local differential privacy (LDP) provides robust
individual privacy by adding noise before data is
shared~\cite{duchi2013localdp}. While helpful in obtaining coarse statistics, LDP
disrupts the fine-grained structure necessary for explainability, making it
unsuitable for SHAP-based device
attribution~\cite{wang2019ldp}. Existing privacy-preserving XAI methods in
sectors such as healthcare and finance aim to mask individual
contributions~\cite{slack2020dpXAI}.\footnote{However, none address the unique
needs of distributed energy intelligence or the requirement for hierarchical
explainability. These requirements differ from other domains because
distributed energy systems must reconcile fine-grained, device-level reasoning
with strict privacy constraints and multi-level aggregation across households
and grid regions.}

 DP has been widely utilised in statistical
aggregation, smart meter summarization, and distributed load
forecasting~\cite{erlingsson2014rappor, acs2011smartmeterdp}. However,
\textbf{(i) DP aggregation typically yields scalar or coarse-level summaries,
not \textit{explainable} outputs. (ii) DP methods assume static privacy budgets
and do not address repeated queries. (iii) No DP framework to date supports
\textit{SHAP-level explainability} with controlled leakage across a
hierarchy.} Edge-based and hierarchical architectures~\cite{edgebox, databox_pds} enable
privacy-sensitive processing close to the data source. Several studies propose
smart grid aggregation schemes~\cite{fog_iot_energy} and zonal-level
forecasting using HE. However, these systems: provide no mechanisms for
\textbf{local explainability}, lack \textbf{differentially private explanation
aggregation}, assume static privacy constraints, and do not integrate
\textbf{dynamic privacy-budget negotiation}. To the best of our knowledge, existing architectures do not jointly provide (i) noise-free, explainability-rich household reports and (ii) differentially private, appliance-aware zonal summaries suitable for distributed energy analysis. Table~\ref{tab:comparison_updated} summarizes how existing methods differ from HXAI.

\begin{table*}[htbp]
\centering
\caption{Comparison of HXAI with Existing Privacy and XAI Frameworks}
\label{tab:comparison_updated}
\scriptsize
\begin{adjustbox}{width=\linewidth}
\begin{tabular}{lcccccc}
\toprule
\textbf{Method} & \textbf{XAI} & \textbf{Privacy} & \textbf{Hardware} & \textbf{Query} & \textbf{System} & \textbf{Model} \\
& \textbf{Level} & \textbf{Mechanism} & \textbf{Security} & \textbf{Control} & \textbf{Utility} & \textbf{Ownership} \\
\midrule
SHAP / LIME \cite{lundberg2017shap, ribeiro2016lime} & Local/Post-hoc & None & No & Static & High & Centralized \\
Federated Learning \cite{mcmahan2017fl, li2020fedxai} & Global/Partial & Encryption/HE & No & Static & Medium & Global/Shared \\
Local DP \cite{duchi2013localdp, wang2019ldp} & None & Perturbation & No & Static & Low & Decentralized \\
Smart Grid DP \cite{acs2011smartmeterdp, efthymiou2010privacy} & None & Noise Injection & No & Static & Medium & Decentralized \\
Privacy Box/TEE \cite{databox_pds, schneider2022tee} & Local Only & Isolation & Yes (TEE) & None & Low & Local Only \\
\midrule
\textbf{HXAI (Proposed)} & \textbf{Edge-Native} & \textbf{DP + Isolation} & \textbf{Yes (TEE)} & \textbf{Dynamic} & \textbf{High} & \textbf{Decentralized} \\
\bottomrule
\end{tabular}
\end{adjustbox}
\end{table*}

Our research expands the concept of DP into the area of explainability and introduces a new query-level leakage bound, which is vital for practical applications in grid operations. The framework, HXAI, offers a unique combination of features: local, noise-free appliance-level explainability; differential privacy for zonal explainability; dynamic privacy-budget negotiation; and cumulative leakage control for multi-query analytics. This comprehensive approach enables privacy-safe and explainable analysis of energy demand in distributed systems, addressing a gap in existing work related to planning, monitoring, and diagnostic decision support, for example, identifying the dominant contributors to recurring evening demand peaks at the neighborhood level without exposing individual household behavior, rather than real-time grid control. To formalize the deployment assumptions and the privacy risks considered in this work, we next describe the system architecture and the corresponding threat model. For clarity, Table~\ref{tab:notation} summarizes the notation used throughout the system model, theoretical analysis, and experimental evaluation.
\begin{table}[t]
\centering
\footnotesize
\setlength{\tabcolsep}{4pt}
\caption{Notation used throughout HXAI}
\label{tab:notation}
\begin{tabular}{@{}ll@{}}
\toprule
\textbf{Symbol} & \textbf{Meaning} \\
\midrule
$i \in \{1,\dots,H\}$ & Household (data owner) index \\
$k \in \{1,\dots,d\}$ & Appliance / device category index \\
$t \in \{1,\dots,T\}$ & Time index \\

$D_{\mathrm{local},i}$ & Local multivariate time series of household $i$ \\
$x_{i,k,t}$ & Consumption of appliance $k$ at time $t$ in household $i$ \\

$f_{\mathrm{local},i}$ & Local prediction model of household $i$ \\
$\widehat{Y}_{i,t}$ & Predicted demand of household $i$ at time $t$ \\

$\phi_{i,k,t}$ & SHAP attribution of appliance $k$ at time $t$ \\
$\phi_{i,t} \in \mathbb{R}^d$ & Vector of appliance-level SHAP attributions \\
$E_{i,t}$ & Model-weighted explanation at time $t$ \\

$S_i \in \mathbb{R}^d$ & Household SHAP summary over time \\
$\widetilde{S}_i$ & Differentially private household explanation \\
$\eta_i$ & Laplace noise vector \\
$\Delta S$ & $\ell_1$-sensitivity of $S_i$ \\

$S_{\mathrm{zonal}}$ & Aggregated zonal explanation vector \\
$S_{\mathrm{true}}$ & Non-private (oracle) zonal explanation \\
$U$ & Zonal explanation utility (cosine similarity) \\

$q_j$ & $j$-th explainability query \\
$\varepsilon_{i,j}^{\mathrm{adjust}}$ & Privacy cost of query $q_j$ for household $i$ \\
$L_i^{\mathrm{total}},\, L_i^{\max}$ & Cumulative and maximum privacy leakage \\

$\mathrm{ICA}_u$ & Information Control Agreement of entity $u$ \\
$\mathcal{Q}_u$ & Queries permitted under $\mathrm{ICA}_u$ \\
\bottomrule
\end{tabular}
\end{table}

\section{System Model and Threat Model}
We consider a distributed smart-grid environment composed of $H$ households, a
zonal edge server operated by the electricity provider, and an optional cloud
backend used only for non-sensitive aggregation~\cite{fang2012smartgrid, palensky2011demand}.
The objective is to enable device-level explainability for grid-load balancing
while ensuring that no household reveals private behavioural patterns or
appliance routines~\cite{efthymiou2010privacy}. In this section we formalize the data
components, entities, communication channels, and privacy assumptions of the
HXAI architecture.

\subsection{System Entities and Trust Assumptions}
\label{sec:system-model}

The HXAI framework consists of multiple interacting entities that collectively
enforce explainability, privacy preservation, and controlled data access across
distributed energy environments~\cite{mortier2016personal}. We explicitly define
each entity, its responsibilities, and the associated trust assumptions.

\paragraph{\textbf{Querying Entity}}
A \emph{Querying Entity} $u \in \{1,\dots,U\}$ denotes an authorized external stakeholder, such as
an energy service provider, grid operator, electricity provider, or accredited
analyst, that issues explainability queries over aggregated household-level
data~\cite{albadi2007demand, palensky2011demand}. Querying entities operate in a
planning, monitoring, and diagnostic capacity, for example by analyzing
contributors to recurring demand peaks, assessing behavioral responses to
tariff policies, forecasting localized demand trends, or evaluating
appliance-upgrade programs. Querying entities do not have direct access to raw
consumption data or noise-free local explanations and interact with the system
exclusively through formally specified, ICA-governed queries. Queries may be
general, yielding zonal summaries, or more specific, requesting appliance-level
signals subject to higher privacy cost. Querying entities are assumed to be
\emph{honest-but-curious}: they follow prescribed protocols but may attempt to
infer additional information from received outputs~\cite{goldreich1998secure}.

\paragraph{\textbf{Household}}
A \emph{Household} $i \in \{1,\dots,H\}$ denotes a residential unit equipped with
smart energy devices and a local explainability module operating within a
private environment. Households retain full access to their own raw consumption
data and noise-free local explanations, which are never directly exposed outside
the household boundary. Each household specifies privacy constraints governing
the release of aggregated or explained information and is treated as a protected
entity in the system. Depending on the application, a household may also act as a querying entity—for
example, when requesting explainability summaries over a local neighborhood or
community—subject to the same authentication and privacy controls as external
querying entities.

\paragraph{\textbf{Authenticator.}}
The \emph{Authenticator} is responsible for identity verification, access
control, and session management~\cite{ferraiolo2001rbac}. It validates credentials
and ensures that only authorized \emph{querying entities} $u$ may initiate
explainability queries. The Authenticator does not access private data or
explanation outputs; its role is limited to verifying identities and binding
authenticated querying entities to their corresponding Information Control
Agreements (ICAs). We assume the Authenticator to be trusted for correctness but
not for data confidentiality.

\paragraph{\textbf{Information Control Agreement (ICA).}}
An \emph{Information Control Agreement (ICA)} is a formal, machine-explainable
contract that governs the execution of explainability queries issued by an
authorized \emph{querying entity}. Let $u \in \{1,\dots,U\}$ index the set of
authorized querying entities. Each querying entity $u$ is associated with an
ICA, represented as
\begin{equation}
    \medmath{\mathrm{ICA}_u =
    \big( \mathcal{Q}_u,\; \varepsilon_u^{\max},\; \mathcal{O}_u,\; T_u \big).}
\end{equation}

Here, $\mathcal{Q}_u$ specifies the permitted query types and scope (e.g.,
general or specific queries), $\varepsilon_u^{\max}$ denotes the maximum
cumulative privacy budget allocated to querying entity $u$, $\mathcal{O}_u$
defines the allowed output granularity, and $T_u$ is the validity period of the
agreement. The ICA specifies an upper bound $\varepsilon_u^{\max}$ on the
cumulative privacy loss that may be incurred through queries issued by $u$,
where smaller values of $\varepsilon$ correspond to stronger privacy
guarantees. Individual queries $q \in \mathcal{Q}_u$ incur a privacy cost
$\varepsilon(q)$, determined by the query’s granularity and potential
information leakage, and are charged against this bound. In addition, each household $i \in \{1,\dots,H\}$ is treated as a protected
entity that may impose its own privacy constraint by specifying a maximum
permissible privacy loss $\varepsilon_i^{\max}$. A query is executed only if its
cumulative privacy cost satisfies both the querying entity’s privacy constraint
$\varepsilon_u^{\max}$ and the privacy constraints of all affected households.
General queries incur lower privacy cost, while specific appliance-level queries
consume a higher privacy budget~\cite{dwork2014dp}. This structure enables
fine-grained, query-aware privacy accounting under the ICA. ICAs are established through an explicit negotiation phase, during which
privacy--utility trade-offs are adjusted within predefined system-level and
household-level constraints~\cite{li2017privacytradeoff}. Once finalized, an ICA
serves as a binding policy artifact that constrains all subsequent query
executions associated with querying entity $u$ and enables cumulative
query-level privacy accounting. Each ICA is bound to a unique household–Private
Data Box pair, ensuring that privacy budgets are enforced independently for
each household.

\paragraph{\textbf{Data Dealer}.}
The \emph{Data Dealer} acts as a policy gatekeeper and compliance authority,
analogous to data governance and consent-management components in privacy-aware
systems~\cite{spiekermann2015privacy}. It evaluates negotiated Information
Control Agreements (ICAs) and determines whether they satisfy regulatory,
contractual, and system-level constraints before approval. The Data Dealer does
not execute queries or process household data directly; instead, it ensures that
only ICAs conforming to organizational policies and cumulative privacy-loss
constraints are admitted for execution. We assume the Data Dealer to be trusted
for policy enforcement but not for access to private data or explanations.

\paragraph{\textbf{Private Data Box}.}
Each household $i \in \{1,\dots,H\}$ is equipped with a \emph{Private Data Box},
which hosts local prediction models and explanation mechanisms within a private
execution environment~\cite{mortier2016personal}. The Private Data Box is
responsible for local model inference, explanation generation, and enforcement
of household-specific privacy constraints specified by approved ICAs. In this
work, the Private Data Box is treated as a \emph{trusted black-box component}.
We assume the availability of industry-standard trusted execution environments
(TEEs)~\cite{costan2016sgx, schneider2022tee, mdpi_riscv_tee, menetrey2022attestation}
that provide confidentiality of household data, integrity of executed code, and
controlled release of differentially private outputs.

\paragraph{\textbf{Zonal Aggregator}.}
The \emph{Zonal Aggregator} collects explanation summaries released by multiple
Private Data Boxes and produces aggregated, zonal-level explanations under
differential privacy~\cite{erlingsson2014rappor}. Aggregation is performed
exclusively over privacy-preserving outputs that comply with the active ICAs,
ensuring that raw consumption data and noise-free household-level explanations
never leave the household boundary. The Zonal Aggregator is assumed to be
untrusted with respect to household data and explanations: it is not granted
access to raw data or noise-free outputs and operates only on differentially
private summaries. Correct noise calibration and privacy enforcement are
performed at the household level within the Private Data Boxes, while the Zonal
Aggregator is assumed only to execute the specified aggregation.

\subsection{\textbf{Trust Boundaries and Threat Model}}

HXAI enforces a strict separation between trusted and untrusted components.
Household-level data and noise-free explanations remain confined within Private
Data Boxes, while only differentially private summaries that satisfy
household-imposed constraints and approved ICAs are exposed beyond this
boundary~\cite{dwork2006dp}. Trusted components (Authenticator, Data Dealer, and
Private Data Boxes) are assumed to execute correctly and never release raw data
or noise-free explanations. HXAI adopts a semi-honest (honest-but-curious) threat
model consistent with the differential privacy literature. Adversaries—including
querying entities and aggregators—may observe all differentially private zonal
outputs, combine reports across households, and issue repeated queries to infer
sensitive household information or re-identify households through cumulative
leakage. Communication is assumed to be encrypted, and differential privacy noise
is applied before any information leaves the household boundary. Attacks on
hardware security primitives and side-channel leakage are outside the scope of
this work. Under these assumptions, HXAI enforces event-level differential
privacy within households while providing privacy guarantees at the zonal
interface.

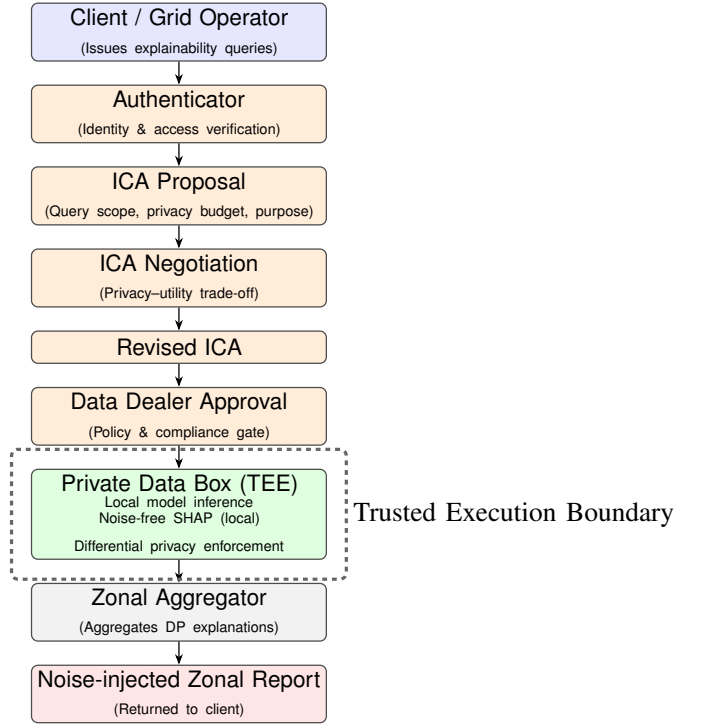
\begin{figure}[htbp]
\centering
\begin{adjustbox}{width=\linewidth}
\begin{tikzpicture}[
    node distance = 0.5cm,
    font=\sffamily\Large, 
    base/.style = {
        rectangle, rounded corners,
        draw=black!70, thick,
        text width=6.2cm,
        align=center,
        minimum height=2em
    },
    actor/.style = {base, fill=blue!10},
    control/.style = {base, fill=orange!15},
    secure/.style = {base, fill=green!12},
    process/.style = {base, fill=gray!10},
    output/.style = {base, fill=red!10},
    arrow/.style = {->, >=Stealth, thick}
]

\node[actor] (client) {Client / Grid Operator\\\small (Issues explainability queries)};
\node[control, below=of client] (auth) {Authenticator\\\small (Identity \& access verification)};

\node[control, below=of auth] (ica_prop) {ICA Proposal\\\small (Query scope, privacy budget, purpose)};
\node[control, below=of ica_prop] (ica_neg) {ICA Negotiation\\\small (Privacy--utility trade-off)};
\node[control, below=of ica_neg] (ica_final) {Revised ICA};

\node[control, below=of ica_final] (dealer) {Data Dealer Approval\\\small (Policy \& compliance gate)};

\node[secure, below=of dealer] (tee) {Private Data Box (TEE)\\
\small Local model inference\\
\small Noise-free SHAP (local)\\
\small Differential privacy enforcement};

\node[process, below=of tee] (agg) {Zonal Aggregator\\\small (Aggregates DP explanations)};
\node[output, below=of agg] (report) {Noise-injected Zonal Report\\\small (Returned to client)};

\draw[arrow] (client) -- (auth);
\draw[arrow] (auth) -- (ica_prop);
\draw[arrow] (ica_prop) -- (ica_neg);
\draw[arrow] (ica_neg) -- (ica_final);
\draw[arrow] (ica_final) -- (dealer);
\draw[arrow] (dealer) -- (tee);
\draw[arrow] (tee) -- (agg);
\draw[arrow] (agg) -- (report);

\node[
    draw=black!60,
    dashed,
    line width=2pt,
    rounded corners,
    inner sep=12pt,
    fit=(tee),
    label={[font=\LARGE, rotate=360]right:Trusted Execution Boundary}
] {};

\end{tikzpicture}
\end{adjustbox}
\caption{End-to-end HXAI workflow. Queries are authorized via an Information Control Agreement (ICA) and executed inside a trusted Private Data Box. Local explanations remain noise-free, while only differentially private summaries are aggregated at the zonal level.}
\label{fig:hxai-system}
\end{figure}


\subsection{\textbf{Data Model}}

We consider a distributed setting in which each household
$i \in \{1,\dots,H\}$ continuously records a multivariate time series of energy
measurements from IoT-enabled devices,
\begin{equation}
    \medmath{D_{\mathrm{local},i}
    = \{x_{i,t} \in \mathbb{R}^{d} \mid t = 1,\dots,T\},}
\end{equation}
where $d$ denotes the number of device categories and $T$ the number of observed
time points. We do not assume that all households possess the same set of
devices. Instead, each household observes a household-specific subset of
appliances. For notational convenience and to enable aggregation, all local
explanations are embedded into a common $d$-dimensional device-category space,
where devices not present in a given household are represented by zeros. This
representation allows heterogeneous households to be aggregated consistently
at the zonal level without revealing device ownership. Using this local data,
each household trains a predictive model
\begin{equation}
   \medmath{\widehat{Y}_{\mathrm{local},i}
   = f_{\mathrm{local},i}(D_{\mathrm{local},i}),}
\end{equation}
which forecasts short-term electricity demand, such as next-hour or next-day
consumption. To explain these predictions, appliance-level SHAP attributions are
computed locally. For a given household $i$ at time $t$, the local SHAP
explanation is defined as the vector
\begin{equation}
    \medmath{
    \mathrm{SHAP}_{\mathrm{local},i}(x_{i,t})
    =
    \big(
        \phi_{i,1,t} x_{i,1,t},
        \dots,
        \phi_{i,d,t} x_{i,d,t}
    \big)
    \in \mathbb{R}^d,}
\end{equation}
where the term $\phi_{i,k,t} x_{i,k,t}$ represents the model-weighted explanatory
contribution of device $k$ to the predicted consumption at time $t$. Here,
$x_{i,k,t}$ denotes the observed energy usage of device $k$, while the SHAP
coefficient $\phi_{i,k,t}$ captures the relevance assigned to that device by the
predictive model in the given context. Their product therefore quantifies the
explanatory impact of device $k$ on the prediction, rather than raw consumption
alone. These explanations remain noise-free and are accessible only within the
household’s Private Data Box.

We emphasize that raw device-level consumption $x_{i,t}$ and SHAP attributions
serve fundamentally different purposes. While $x_{i,t}$ reports observed energy
usage per device, SHAP attributions explain how each device contributes to the
model’s prediction at time $t$, accounting for feature interactions,
correlations, and nonlinear model behavior. As such, SHAP explains the
predictive mechanism rather than the measurement itself.\footnote{Local SHAP
explanations enable households to understand which appliances actively drive
predicted demand peaks or price-relevant forecasts, as opposed to devices that
merely consume energy without influencing the prediction. For example, a
household may observe that a washing machine has lower average energy
consumption than electric heating, yet receives a higher SHAP contribution
during peak-price hours because its operation coincides with critical demand
periods. This distinction cannot be inferred from raw consumption alone, but is
immediately visible through SHAP-based explanations. Such insights allow
households to take informed actions, such as rescheduling appliance usage,
adjusting operating durations, or prioritizing efficiency upgrades for devices
with the highest explanatory impact on demand- or price-sensitive forecasts.}
Finally, SHAP provides a normalized, additive explanation representation that can
be meaningfully aggregated across households. This property is essential for
constructing zonal-level explanations under differential privacy, whereas raw
consumption data cannot be safely or meaningfully aggregated for explainability
without revealing sensitive behavioral patterns.

\subsection{\textbf{Hierarchical Aggregation Model}}

At the zonal level, individual household explanations are never revealed.
Instead, zonal explanations are constructed by aggregating differentially
private local SHAP values,
\begin{equation}
    \medmath{
    \mathrm{SHAP}_{\mathrm{zonal}}
    =
    \sum_{i=1}^{H}
    \left(
        \sum_{t=1}^{T}
        \mathrm{SHAP}_{\mathrm{local},i}(x_{i,t})
        + \eta_i
    \right),}
\end{equation}
where $\eta_i \in \mathbb{R}^d$ denotes a Laplace noise vector calibrated
according to the household’s negotiated privacy loss parameter
$\varepsilon_i^{\mathrm{adjust}}$. This hierarchical structure ensures that
explainability at the household level remains noise-free, while explainability
at the zonal level is protected by differential privacy, and only aggregated
explanations contribute to downstream analysis. Each household $i$ selects an
initial privacy constraint $\varepsilon_i^{\mathrm{init}}$ and receives a
utility--risk estimate derived from the anticipated explainability output. We
distinguish clearly between \emph{households} and \emph{querying entities}. A
household $i$ denotes a data owner equipped with a Private Data Box and
associated privacy constraints, whereas a querying entity $u$ refers to an
authorized external actor (e.g., an energy service provider) that issues
explainability queries under an Information Control Agreement (ICA). A
negotiation mechanism determines an adjusted household privacy level
\begin{equation}
    \medmath{\varepsilon_i^{\mathrm{adjust}}
    \in (\varepsilon_i^{\mathrm{opt}} - \ell,\;
          \varepsilon_i^{\mathrm{opt}}),}
\end{equation}
where $\varepsilon_i^{\mathrm{opt}}$ denotes the household-specific privacy loss
level that maximizes explanation utility subject to declared privacy
preferences. In practice, this value is obtained via bounded search or
feedback-based negotiation, where the household evaluates the utility--risk
trade-off of candidate privacy levels proposed by the querying entity. The
relaxation parameter $\ell$ defines a small admissible interval around
$\varepsilon_i^{\mathrm{opt}}$ to allow flexibility under repeated queries and
to ensure stable operation over time. Importantly, participation in the
negotiation process is voluntary and bounded by household-defined privacy
limits. A household is never forced to accept a weaker privacy guarantee than
its stated preference; if no agreement is reached, the system defaults to
stricter privacy with reduced explanatory utility.

\subsection{\textbf{Query Model}}
We distinguish between two classes of queries issued by querying entities:
\textbf{General Queries ($q^{\mathrm{gen}}$)} aggregate zonal summaries (e.g.,
demand trends, peak shapes, dominant appliance categories) that incur a lower
privacy cost than appliance-level queries, while \textbf{Specific Queries
($q^{\mathrm{spec}}$)} perform appliance-level diagnostics, behavioral
elasticity analysis, and evaluation of incentive effects. Due to their finer
granularity, these queries incur a higher privacy cost $\varepsilon$. For a
sequence of $n$ queries affecting household $i$, the cumulative privacy loss is
defined as
\begin{equation}
    \medmath{
    L_i^{\mathrm{total}} = \sum_{j=1}^{n} \varepsilon_{i,j}^{\mathrm{adjust}}.}
\end{equation}
Once the cumulative privacy loss $L_i^{\mathrm{total}}$ reaches or exceeds the
negotiated threshold, further queries affecting household $i$ are automatically
denied by the Private Data Box.



\section{Methodology and Theoretical Foundations}

HXAI integrates household-level explainability, differentially private zonal
aggregation, and multi-query privacy accounting into a unified hierarchical
architecture. This section formalizes the computational pipeline, introduces
theoretical guarantees for privacy-preserving explainability, and describes the
differential privacy mechanisms, aggregation rules, and privacy-budget
negotiation strategy used to enable explainable and privacy-safe analysis of
distributed energy demand. Throughout the analysis, we assume bounded
SHAP attribution vectors with finite $\ell_1$ norm and non-degenerate zonal
explanations ($\|S_{\mathrm{true}}\|_2 \ge c > 0$). This non-degeneracy assumption
excludes trivial cases in which all appliance-level contributions vanish at the
zonal level and holds whenever at least one appliance exhibits non-negligible
explanatory influence on aggregate demand, which is typical in realistic energy
systems. All expectations are taken with respect to the randomness of the
Laplace mechanism, conditioned on the observed household datasets. We begin at the household level, where each household $i$, for
$i \in \{1,\dots,H\}$, models and explains its own energy consumption using local
data and computation within its Private Data Box.

\begin{definition}[\textbf{Local Explainability}]
Let $f_{\mathrm{local},i}$ denote the prediction model of household $i$ and let
\begin{equation}
    \medmath{\widehat{Y}_{i,t} = f_{\mathrm{local},i}(x_{i,t}),}
\end{equation}
where $x_{i,t} \in \mathbb{R}^d$ represents appliance-level readings at time $t$.
Local explainability is achieved by computing SHAP attributions
\begin{equation}
    \medmath{\phi_{i,k,t} = \mathrm{SHAP}_{f_{\mathrm{local},i}}(x_{i,k,t}),}
\end{equation}
where $\phi_{i,k,t}$ captures how strongly the predictive model attributes
relevance to appliance $k$ at time $t$, accounting for feature interactions and
temporal context. Multiplying by the observed consumption yields the
model-weighted explanatory contribution of appliance $k$ to the prediction. The model-weighted explanation at time $t$ is given by,
\begin{equation}
    \medmath{E_{i,t} = \sum_{k=1}^{d} \phi_{i,k,t} \, x_{i,k,t}.}
\end{equation}
\end{definition}
Thus, $E_{i,t}$ provides an additive decomposition of the \emph{predicted}
consumption at time $t$ into appliance-level \emph{explanatory contributions},
answering the question: which appliances most strongly influenced the model’s
demand prediction, and by how much. Importantly, these contributions reflect
model attribution rather than raw energy usage.
All explanations remain noise-free and confined within the household’s Private
Data Box, supporting local interpretability while enabling downstream
privacy-preserving aggregation. The additive structure of the explanation is
crucial: it allows appliance-level attributions to be meaningfully aggregated
across households under differential privacy, whereas raw consumption values
lack a compositional, model-aware explanation. All computations
in Definition~1 are executed locally under the constraints of the approved
ICA.
Throughout the remainder of this paper, and with slight abuse of notation,
$\phi_{i,k,t}$ denotes the model-weighted contribution
$\phi_{i,k,t}\,x_{i,k,t}$ whenever aggregation is performed.

\begin{definition}[\textbf{Differentially Private Household Explanation}]
Let $\phi_{i,t} = (\phi_{i,1,t}, \ldots, \phi_{i,d,t})$ denote the appliance-level
SHAP attribution vector of household $i$ at time $t$. The household summary
statistic is defined as
\begin{equation}
    \medmath{S_i = \sum_{t=1}^{T} \phi_{i,t},}
\end{equation}
where $S_i \in \mathbb{R}^d$ summarizes the cumulative explanatory influence of
each appliance over time. A differentially private household explanation is
obtained via the Laplace mechanism,
\begin{equation}
    \medmath{\widetilde{S}_i = S_i + \eta_i,}
\end{equation}
where $\eta_i \in \mathbb{R}^d$ is a noise vector with independent components
\begin{equation}
    \medmath{\eta_{i,k} \sim \mathrm{Lap}\!\left(\frac{\Delta S}{\varepsilon_i^{\mathrm{adjust}}}\right),
    \qquad k = 1, \ldots, d.}
\end{equation}
Here, $\Delta S$ denotes the $\ell_1$-sensitivity of the household summary, and
$\varepsilon_i^{\mathrm{adjust}}$ is the negotiated and ICA-approved privacy
budget enforced locally by the Private Data Box.
\end{definition}
For notational simplicity, $\phi_{i,t}$ is used to denote the model-weighted
appliance-level explanation vector, absorbing the multiplication by
$x_{i,k,t}$. The vector $S_i$ encodes which appliances systematically influence the
household’s predicted demand, while $\widetilde{S}_i$ ensures this information
can be shared externally without violating household privacy. Having established household-level explainability, we now formalize how explanations from multiple households are combined to produce privacy-preserving insights at the zonal level.
The formal derivation of the $\ell_1$-sensitivity $\Delta S$ and its dependence
on bounded SHAP attributions is provided in Appendix~\ref{app:sensitivity}.
\begin{definition}[\textbf{Privacy--Utility Negotiation Objective}]
For each household $i$, privacy-budget negotiation is modeled as the
scalarized optimization problem
\begin{equation}
    \medmath{\max_{\varepsilon_i} \; U(\varepsilon_i) - \lambda_i \varepsilon_i,}
\end{equation}
where $U(\varepsilon_i)$ denotes explanation utility at privacy level
$\varepsilon_i$, and $\lambda_i$ captures household-specific privacy preference.
In practice, this objective is solved heuristically through bounded search and
feedback, but the formulation provides a principled basis for negotiation.
\end{definition}

\begin{definition}[\textbf{Zonal Explanation}]
Let $\widetilde{S}_i$ denote the differentially private explanation vector
released by household $i$. The zonal explanation is defined as
\begin{equation}
    \medmath{S_{\mathrm{zonal}} = \sum_{i=1}^{H} \widetilde{S}_i,}
\end{equation}
yielding an appliance-aware, differentially private summary of regional demand
behavior.
\end{definition}
The Zonal Aggregator operates exclusively on differentially private household
summaries and never accesses raw data or noise-free local explanations. While
differential privacy does not preclude all forms of inference, it provides a
formal bound on the information that can be inferred about any individual
household. In HXAI, this risk is further mitigated by (i) aggregation across
multiple independently perturbed households and (ii) strict query-aware
privacy-budget enforcement, which prevents adversaries from averaging out noise
through repeated queries. Consequently, any attempted reverse inference remains
provably bounded and cannot exceed the negotiated privacy guarantees.

\begin{theorem}[\textbf{Unbiasedness of Aggregated DP Explanations}]
Let $\widetilde{S}_i$ be differentially private household summaries generated
using the Laplace mechanism. Then
\begin{equation}
    \medmath{S_{\mathrm{zonal}} = \sum_{i=1}^{H} \widetilde{S}_i}
\end{equation}
is an unbiased estimator of the true zonal SHAP contribution
$\sum_{i=1}^{H} S_i$.
\end{theorem}
\begin{proof}
Each summary is released as $\widetilde{S}_i = S_i + \eta_i$, where $\eta_i$ has
zero mean. Taking expectation,
\begin{equation}
    \medmath{\mathbb{E}[S_{\mathrm{zonal}}]
= \sum_{i=1}^{H} \mathbb{E}[S_i + \eta_i]
= \sum_{i=1}^{H} S_i,}
\end{equation}
which establishes unbiasedness.
\end{proof}
Although privacy noise perturbs individual household summaries, zonal
explanations remain correct in expectation, ensuring that planning,
monitoring, and diagnostic decisions based on aggregated explanations are not
systematically biased by differential privacy.
\begin{theorem}[\textbf{Error Bound of Aggregated DP-SHAP}]
The mean absolute deviation (MAD) of the zonal explanation from the true zonal SHAP sum satisfies
\begin{equation}
    \medmath{\mathbb{E}\!\left[
    \left\| S_{\mathrm{zonal}} - \sum_{i=1}^{H} S_i \right\|_1
\right]
\le
\sum_{i=1}^{H} \frac{d \, \Delta S}{\varepsilon_i^{\mathrm{adjust}}},}
\end{equation}
where $d$ is the number of appliances, $\Delta S$ is the $\ell_1$-sensitivity of
the household summary, and $\varepsilon_i^{\mathrm{adjust}}$ is the negotiated
privacy budget of household $i$.
\label{thm2}
\end{theorem}

\begin{proof}
Each coordinate of $\widetilde{S}_i$ is perturbed by independent Laplace noise with
scale $\Delta S / \varepsilon_i^{\mathrm{adjust}}$.
For Laplace noise, the mean absolute deviation satisfies
\begin{equation}
    \medmath{\mathbb{E}[|\eta|] = \frac{\Delta S}{\varepsilon_i^{\mathrm{adjust}}}.}
\end{equation}
Summing over $d$ coordinates and $H$ households yields the stated bound.
\end{proof}
We assume that the household explanation summary $S_i$ has bounded
$\ell_1$-sensitivity $\Delta S$ with respect to a single time-step change.
This result explicitly characterizes the privacy–utility trade-off: tighter
privacy budgets increase noise, while higher $\varepsilon_i^{\mathrm{adjust}}$
values preserve explanation fidelity. Importantly, the bound scales linearly and
remains predictable, ensuring that zonal explanations remain actionable for
grid-level diagnostics and planning.

\begin{definition}[\textbf{Zonal Explanation Utility}]
Let $S_{\mathrm{true}} = \sum_{i=1}^{H} S_i$ denote the noiseless zonal SHAP
explanation and let $S_{\mathrm{zonal}}$ be its differentially private
counterpart.
The utility of the zonal explanation is defined as the cosine similarity
\begin{equation}
  \medmath{ U =
\frac{
    S_{\mathrm{true}}^\top S_{\mathrm{zonal}}
}{
    \| S_{\mathrm{true}} \|_2 \, \| S_{\mathrm{zonal}} \|_2
}.}
\end{equation}
\end{definition}

This definition captures how well the relative importance of appliances is
preserved under privacy noise. Cosine similarity is chosen because it captures
directional alignment and relative feature importance while being invariant to
scale, making it well suited for evaluating semantic stability under additive
noise.

\begin{theorem}[\textbf{Utility degradation under differential privacy}]
Assume that zonal explanation vectors have bounded norm, i.e.,
$\|S_{\mathrm{true}}\|_2 \ge c > 0$, and that each household summary has bounded
$\ell_1$-sensitivity $\Delta S$. Under the Laplace mechanism, the expected cosine
similarity between the private zonal explanation $S_{\mathrm{zonal}}$ and the
true explanation $S_{\mathrm{true}}$ satisfies
\begin{equation}
    \medmath{\mathbb{E}[U] \ge
    1 - \frac{C}{\|S_{\mathrm{true}}\|_2}
    \sum_{i=1}^{H} \frac{d\,\Delta S}{\varepsilon_i^{\mathrm{adjust}}},}
\end{equation}
for a constant $C > 0$.
\end{theorem}

\begin{proof}
From Theorem~1 and Theorem~2, the expected deviation between
$S_{\mathrm{zonal}}$ and $S_{\mathrm{true}}$ scales linearly with the magnitude
of aggregated Laplace noise.
Since cosine similarity is stable under small additive perturbations, the
expected reduction in similarity is proportional to the relative noise
magnitude, yielding the stated bound. We assume non-degenerate zonal
explanations with bounded norm, which holds in practice when aggregating over
multiple households.
\end{proof}

The above results justify HXAI’s architectural choice to aggregate
explanations prior to privacy injection. Unbiasedness ensures no
systematic distortion, while linear error growth enables predictable
privacy–utility negotiation. Crucially, cosine stability explains the
empirical preservation of semantic structure (rank, sign, dominance)
observed in Section~V, even under strong privacy constraints.
As privacy budgets increase, the relative impact of noise diminishes, and zonal
explanations increasingly align with the true explanation. Hence, the expected
zonal explanation utility $\mathbb{E}[U]$ is monotone non-decreasing in each
$\varepsilon_i^{\mathrm{adjust}}$.

In operational settings, grid operators may issue multiple explanation queries
over time. We therefore formalize how privacy leakage accumulates and how
overexposure is prevented.
By Lemma~\ref{lem} (Appendix) and Theorem~\ref{thm2}, the expected perturbation
norm is bounded by $\sum_i d\Delta S / \varepsilon_i^{\mathrm{adjust}}$,
yielding the stated result.

Let $Q = \{ q_1, q_2, \dots, q_n \}$ denote a sequence of queries issued to
household $i$, where each query $q_j$ consumes a negotiated privacy budget
$\varepsilon_{i,j}^{\mathrm{adjust}}$. We use privacy leakage and cumulative
privacy budget interchangeably, with
$L_i^{\mathrm{total}} \equiv \varepsilon_i^{\mathrm{total}}$ and
$L_i^{\max} \equiv \varepsilon_i^{\max}$.

\begin{definition}[\textbf{Cumulative Privacy Leakage}]
The total privacy leakage incurred by household $i$ after $n$ queries is
\begin{equation}
   \medmath{
   L_i^{\mathrm{total}} =
   \sum_{j=1}^{n} \varepsilon_{i,j}^{\mathrm{adjust}}.}
\end{equation}
\end{definition}

\begin{definition}[\textbf{Leakage Threshold Enforcement}]
A household $i$ denies any further query once
\begin{equation}
    \medmath{L_i^{\mathrm{total}} \ge L_i^{\mathrm{max}},}
\end{equation}
where $L_i^{\mathrm{max}}$ is the household’s negotiated leakage threshold.
\end{definition}

\begin{theorem}[\textbf{Leakage Upper Bound for Arbitrary Query Sequences}]
For any (possibly adaptive) sequence of queries $Q$, the total privacy leakage
from household $i$ is bounded by the negotiated threshold
$L_i^{\mathrm{max}}$. No additional leakage occurs beyond this point.
\end{theorem}

\begin{proof}
By the sequential composition property of differential privacy, privacy loss
accumulates linearly with each query.
Since the household databox halts all responses once
$L_i^{\mathrm{total}} = L_i^{\mathrm{max}}$, no further information is released,
and no additional leakage is possible.
\end{proof}
This result follows directly from the standard sequential composition property
of differential privacy; for completeness, a formal statement is provided in
Appendix~\ref{app:composition} and  Appendix~\ref{app:composition2}. Together, the system model and theoretical framework establish HXAI as a
query-aware, explanation-centric architecture in which privacy guarantees are
enforced structurally rather than through post hoc noise injection.
Algorithm~\ref{alg:hxai} summarizes the end-to-end HXAI execution pipeline.

\begin{algorithm}[htbp]
\small
\caption{HXAI: Hierarchical Privacy-Preserving Explainability}
\label{alg:hxai}
\begin{algorithmic}[1]

\Require
Local household data $D_{\mathrm{local},i}$,
approved ICA $\mathrm{ICA}_u = (\mathcal{Q}_u, \varepsilon_u^{\max}, \mathcal{O}_u, T_u)$,
query sequence $Q = \{q_1,\dots,q_n\}$

\State \textbf{/* Executed inside Private Data Box (TEE) */}
\State Compute local predictions:
$\widehat{Y}_{i,t} \gets f_{\mathrm{local},i}(x_{i,t})$
\State Compute SHAP explanations:
$\phi_{i,k,t} \gets \mathrm{SHAP}_{f_{\mathrm{local},i}}(x_{i,k,t})$
\State Aggregate household summary:
$S_i \gets \sum_{t=1}^{T} \phi_{i,t}$
\State Initialize cumulative leakage:
$\varepsilon_i^{\mathrm{total}} \gets 0$

\For{each query $q_j \in Q$}
    \If{$q_j \notin \mathcal{Q}_u$}
        \State \textbf{deny query}
        \State \textbf{continue}
    \EndIf

    \State Determine privacy cost:
    $\varepsilon_{i,j}^{\mathrm{adjust}} \gets \varepsilon(q_j)$

    \If{$\varepsilon_i^{\mathrm{total}} + \varepsilon_{i,j}^{\mathrm{adjust}} > L_i^{\mathrm{max}}$}
        \State \textbf{deny query}
        \State \textbf{continue}
    \EndIf

    \State Sample noise:
    $\eta_i \sim \mathrm{Lap}\!\left(\Delta S / \varepsilon_{i,j}^{\mathrm{adjust}}\right)$
    \State Release DP summary:
    $\widetilde{S}_i \gets S_i + \eta_i$
    \State Update leakage:
    $\varepsilon_i^{\mathrm{total}} \gets
    \varepsilon_i^{\mathrm{total}} + \varepsilon_{i,j}^{\mathrm{adjust}}$
\EndFor

\State \textbf{/* Executed by Zonal Aggregator */}
\State Zonal aggregation:
$S_{\mathrm{zonal}} \gets \sum_{i=1}^{H} \widetilde{S}_i$
\State \Return $S_{\mathrm{zonal}}$

\end{algorithmic}
\end{algorithm}

\section{Experimental Evaluation and Results}
\label{sec:results}
In this section we evaluates the HXAI framework on synthetic and real-world
datasets under differential privacy (DP) constraints. We aim to see if HXAI can
maintain the \emph{semantic structure} of explanations, such as feature
importance, consistency, and explainability, while ensuring privacy. We
distinguish between \emph{primary results}, which establish the core claims of
HXAI, and \emph{supporting analyses}, which provide deeper insight into
explanation behavior under privacy.
Our focus is on explanation robustness and faithfulness, analyzing how privacy
allocation strategies affect (i) deviation from non-private explanations, (ii)
feature ranking stability, and (iii) overall utility. Results are presented as
mean values with standard deviations across multiple random seeds to ensure
robustness and reproducibility.\footnote{These results focus on explanation
quality and semantic stability under controlled, dataset-level conditions,
isolating the behavior of HXAI’s aggregation and privacy mechanisms independent
of deployment considerations.} In contrast, the following Section
\ref{sec:system_experiments} evaluates HXAI as an end-to-end system, examining
how these explanation properties translate into actionable insights, privacy
negotiation, and decision-making dynamics in realistic multi-stakeholder
settings.
\subsection{Datasets}
\label{sec:datasets}
We evaluate HXAI using three datasets that showcase different attribution
regimes. The UCI Appliances (also commonly referred to as ``Appliances Energy
Prediction'' from the UCI Machine Learning Repository)\footnote{The dataset
contains 19,735 observations of hourly residential energy usage and features a
consistent attribution structure, making it smooth and rich in redundancy.}
dataset contains 19,735 observations of hourly residential energy usage and
features a consistent attribution, making it smooth and rich in redundancy.
The UCI Household Power dataset provides approximately 2 million minute-level
electricity records collected over several years. It exhibits strong
non-stationarity and diverse feature importance, which presents challenges for
uniform privacy allocation. Lastly, the UCI Bike Sharing (Hourly) dataset
includes 17,379 hourly bike rental counts and features moderate correlation
among its 12 variables. Together, these datasets contrast homogeneous
attribution, in which importance is evenly distributed, with heterogeneous
regimes, in which certain features exert a dominant influence.
\subsection{Models, Baselines, and Metrics}
\label{sec:models_metrics}
We train gradient-boosted tree regression models on lagged time-series data to
capture temporal dependencies across various datasets, ensuring a fair
comparison through consistent model architectures. Feature attributions are
calculated using TreeSHAP\footnote{We employ tree-based models with TreeSHAP as
they provide exact Shapley attributions, ensuring that observed distortions
arise solely from privacy mechanisms rather than approximation error.} for exact
Shapley-value-based explanations, with non-private explanations serving as a
reference. We explore three explanation mechanisms: the Central DP-SHAP, which
adds Laplace\footnote{We focus on Laplace-based DP-SHAP as it represents the
dominant baseline for explanation-level differential privacy. Extensions to
Gaussian mechanisms and smooth sensitivity primarily affect noise calibration
rather than the structural limitations addressed by HXAI.} noise to aggregated
explanations for global privacy; HXAI (fixed $\epsilon$), that computes local
explanations independently and aggregates them under a uniform privacy budget;
and HXAI (negotiated $\epsilon$), which adapts privacy budgets based on user
preferences, the needs of electricity companies, attribution mass, and other
factors. This approach allows more important features to receive higher privacy
allocations while still adhering to an overall constraint. The performance of the explanations is evaluated using three metrics: Mean Absolute Error (MAE) measures the average absolute deviation between private zonal explanations and the non-private oracle; Spearman Rank Correlation assesses the stability of feature importance ordering under privacy noise; and Cosine Similarity quantifies the directional alignment between private and oracle explanations, reflecting overall utility. All metrics are evaluated over a logarithmically spaced grid of privacy budgets $\epsilon \in [0.1, 20]$, covering both strict and relaxed
privacy regimes.
\begin{figure*}[htbp]
    \centering
    \begin{minipage}[t]{0.32\linewidth}
        \centering
        \includegraphics[width=\linewidth, height=5cm]{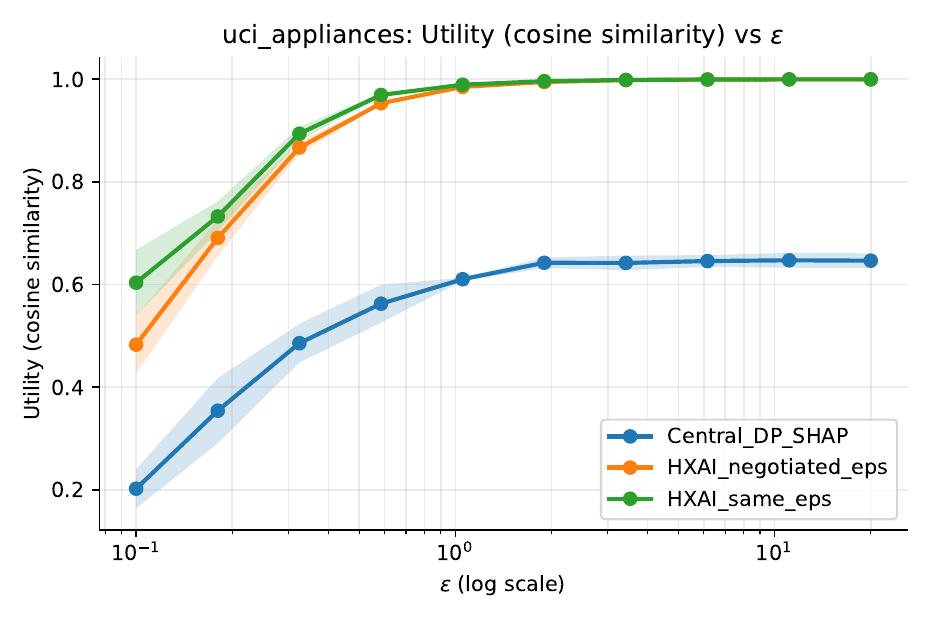}
        \\[-0.5em]
        (a) UCI Appliances
    \end{minipage}
    \hfill
    \begin{minipage}[t]{0.32\linewidth}
        \centering
        \includegraphics[width=\linewidth, height=5cm]{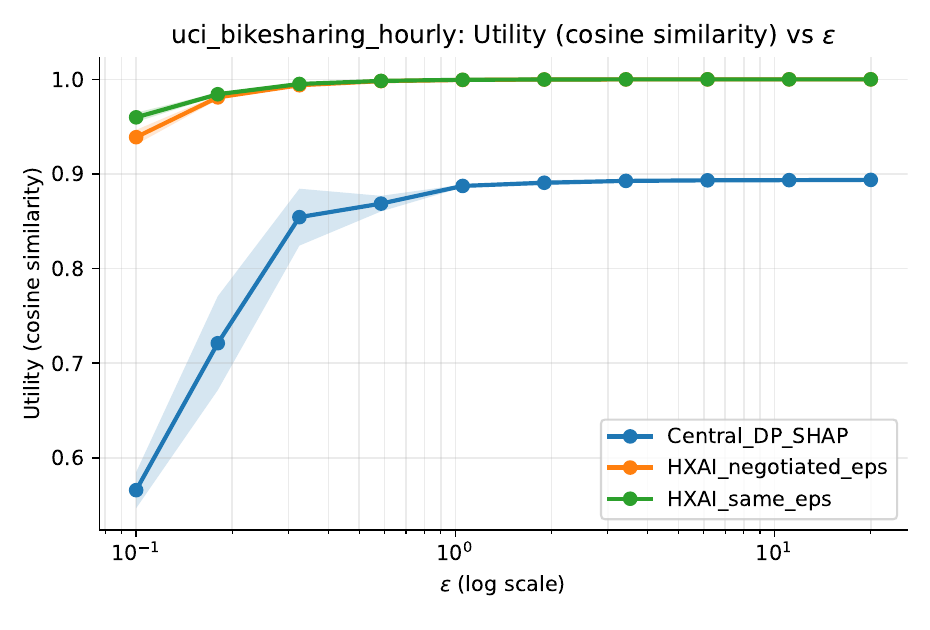}
        \\[-0.5em]
        (b) UCI Bike Sharing (Hourly)
    \end{minipage}
    \hfill
    \begin{minipage}[t]{0.32\linewidth}
        \centering
        \includegraphics[width=\linewidth, height=5cm]{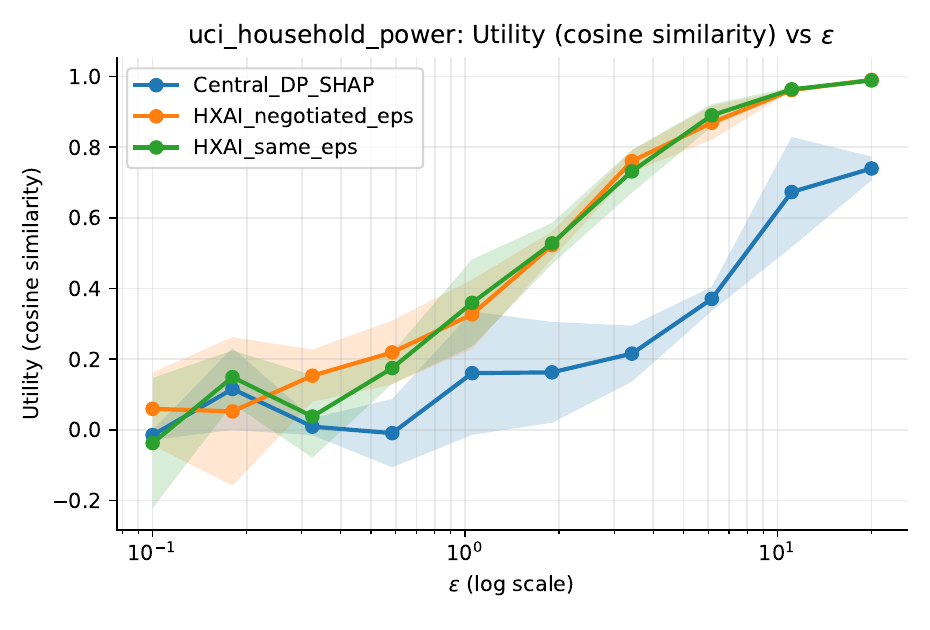}
        \\[-0.5em]
        (c) UCI Household Power
    \end{minipage}

    \caption{Explanation quality as a function of the privacy budget $\epsilon$
    across three real-world datasets. Each panel reports mean absolute error (MAE),
    ranking stability (Spearman), and cosine utility. Shaded regions indicate
    one standard deviation across random seeds.}
    \label{fig:curves_all}
\end{figure*}

\subsection{Results}
\subsubsection{\textbf{Quantitative Results}}
\label{sec:quantitative_results}
Figure~\ref{fig:curves_all} show the evolution of explanation quality with the privacy budget $\epsilon$ across three datasets. We assess mean absolute error (MAE), Spearman rank correlation, and cosine similarity, measuring numerical deviation, ranking stability, and directional utility, respectively. The shaded areas indicate one standard deviation, showing robustness to variance.
\begin{figure*}[htbp]
    \centering
    \begin{minipage}[t]{0.48\linewidth}
        \centering
        \includegraphics[width=\linewidth, height=5cm]{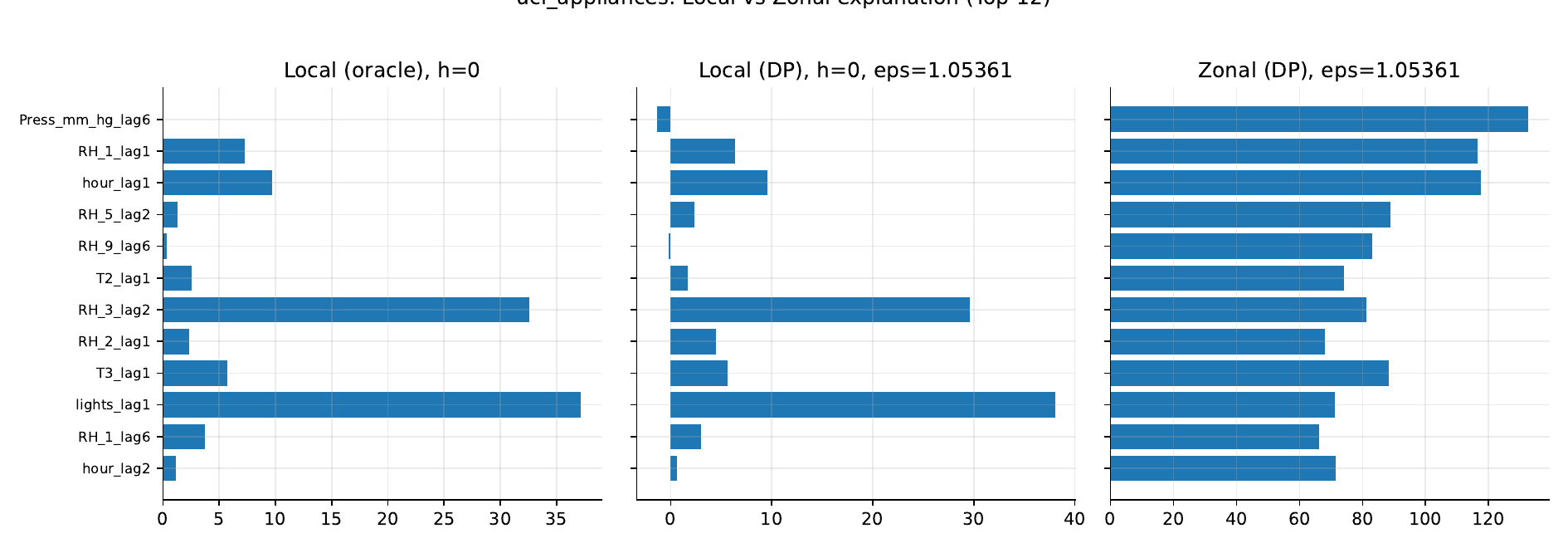}
        \\[-0.6em]
        (a) UCI Appliances
    \end{minipage}
    \hfill
    \begin{minipage}[t]{0.48\linewidth}
        \centering
        \includegraphics[width=\linewidth, height=5cm]{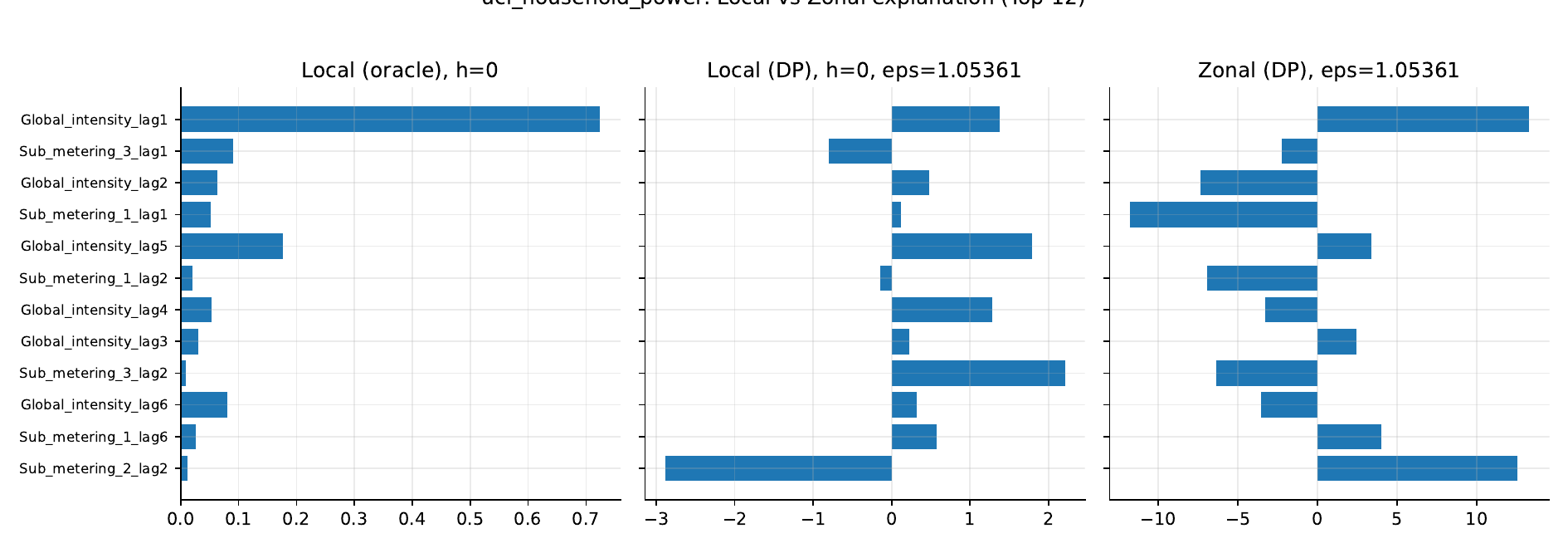}
        \\[-0.6em]
        (b) UCI Household Power
    \end{minipage}

    \caption{Comparison of oracle, local-DP, and zonal-DP SHAP bar plots at
    moderate privacy ($\epsilon \approx 1.05$). Local explanations suffer
    from sign flips and unstable ordering, while zonal HXAI explanations
    preserve dominant features and semantic structure.}
    \label{fig:bars_all}
\end{figure*}

 For the UCI Appliances and UCI Bike Sharing datasets, both HXAI variants quickly converge to the non-private oracle as $\epsilon$ increases (Fig.~\ref{fig:curves_all}(a) and Fig.~\ref{fig:curves_all}(b)), with uniform privacy allocation consistently outperforming negotiated allocation across all metrics. In contrast, the UCI Household Power dataset shows significant variability in feature importance. Fig.~\ref{fig:curves_all}(c) reveals that HXAI with negotiated privacy budgets outperforms uniform allocation, particularly at moderate privacy levels ($\epsilon \leq 2$), showcasing better cosine similarity and ranking stability. This highlights the importance of rank preservation for explainability under differential privacy. Across all datasets, Central DP-SHAP underperforms compared to HXAI variants. Although cosine similarity may improve at high $\epsilon$, the low ranking stability indicates central noise can distort relative importance. Table~\ref{tab:best_results} summarizes top configurations, showing HXAI's consistent superiority. Uniform allocation is optimal for homogeneous settings, while negotiated allocation excels in heterogeneous contexts. Overall, HXAI enhances numerical fidelity while preserving the semantic structure of explanations. We will follow this with a qualitative analysis of how local and zonal explanations evolve under privacy noise.

\begin{table}[htbp]
\centering
\scriptsize
\caption{Best explanation quality achieved per dataset and method
(mean $\pm$ std over seeds).}
\label{tab:best_results}
\begin{tabular}{l l c c c}
\toprule
Dataset & Method & Utility $\uparrow$ & Rank $\uparrow$ & MAE $\downarrow$ \\
\midrule
Appliances & HXAI (same $\epsilon$) & 0.9999 & 0.9991 & 0.25 \\
Bike & HXAI (same $\epsilon$) & 0.9999 & 0.9854 & 0.25 \\
Household & HXAI (neg. $\epsilon$) & 0.9908 & 0.4162 & 0.22 \\
\midrule
Appliances & Central DP & 0.65 & 0.45 & 29.45 \\
Household & Central DP & 0.74 & 0.17 & 0.45 \\
\bottomrule
\end{tabular}
\end{table}

\subsubsection{\textbf{Qualitative Robustness Analysis}}
\label{sec:qualitative}

While quantitative metrics measure numerical similarity and ranking stability, they do not fully capture whether explanations are explainable for human analysts. Small numerical changes can cause sign flips or reordering among top features, leading to misleading explanations. Therefore, we conduct a qualitative robustness analysis to visualize how explanations change under privacy noise. Figures~\ref{fig:bars_all} and~\ref{fig:rankflow_all} illustrate the qualitative behavior of various explanation methods across datasets. Figures~\ref{fig:bars_all} and~\ref{fig:rankflow_all} illustrate the qualitative behavior of explanation methods across datasets. Figure~\ref{fig:bars_all} compares oracle (no privacy), local-DP, and zonal-DP SHAP explanations at two privacy levels, using the oracle as a reference for true feature importance and direction.

\begin{figure*}[htbp]
    \centering
    \begin{minipage}[htbp]{0.48\linewidth}
        \centering
        \includegraphics[width=\linewidth]{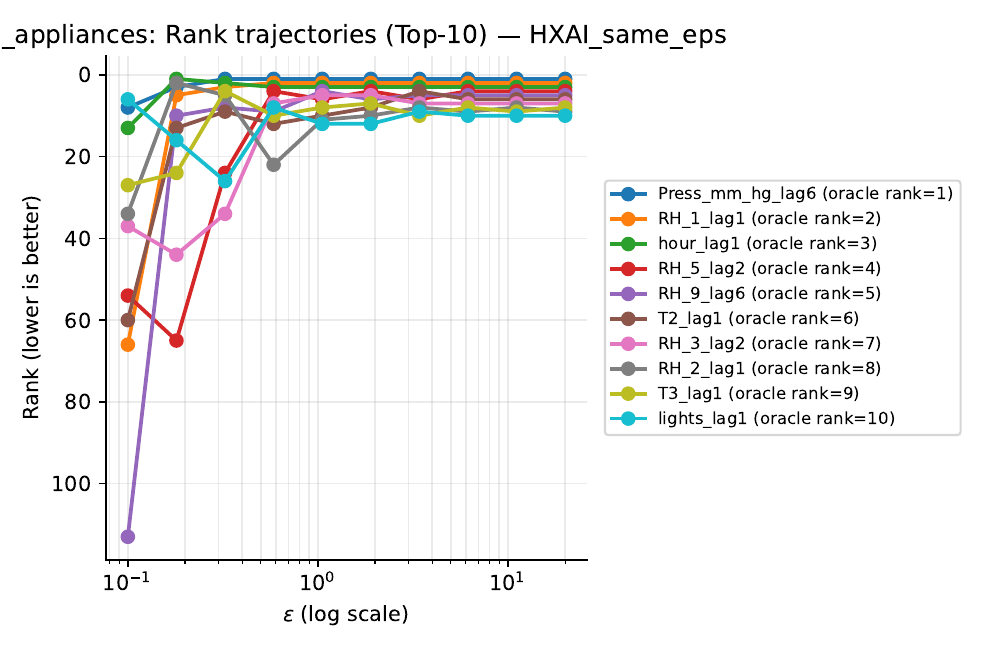}
        \\[-0.6em]
        (a) UCI Appliances
    \end{minipage}
    \hfill
    \begin{minipage}[htbp]{0.48\linewidth}
        \centering
        \includegraphics[width=\linewidth]{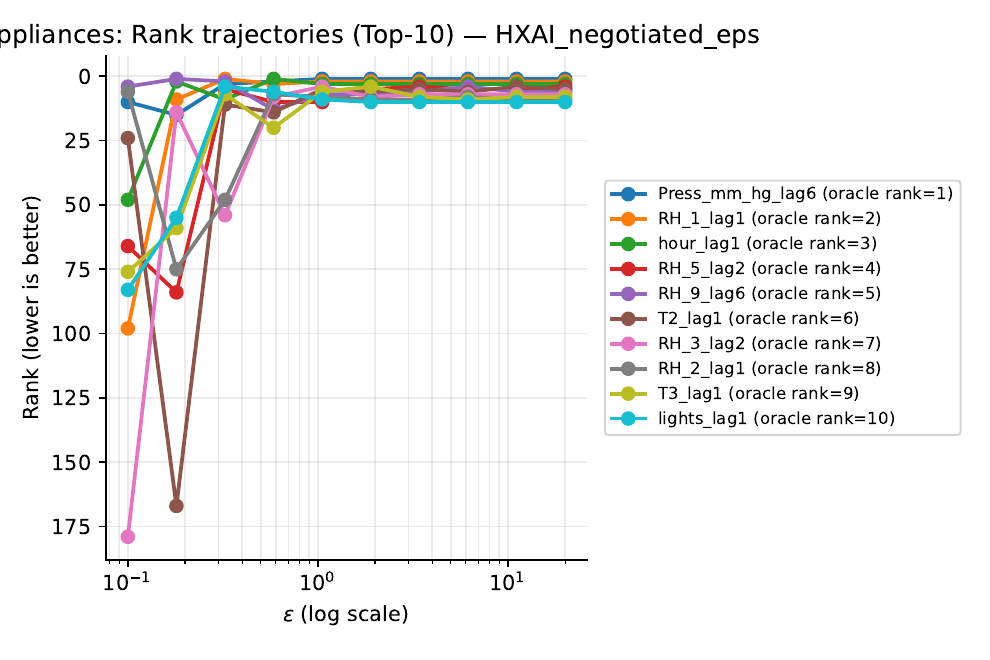}
        \\[-0.6em]
        (b) UCI Household Power
    \end{minipage}

    \caption{Top-10 feature rank-flow trajectories under increasing privacy
noise for HXAI. Near-diagonal trajectories indicate stable semantic
ordering. Negotiated privacy allocation yields tighter rank trajectories
in heterogeneous attribution regimes.}
    \label{fig:rankflow_all}
\end{figure*}

In contrast, zonal HXAI explanations exhibit substantially greater structural
stability under privacy noise. Although differential privacy attenuates the
absolute magnitudes of feature attributions, dominant features remain visually
identifiable and comparatively well ordered, preserving a clear explanatory
signal. For example, in Fig.~\ref{fig:bars_all}(a) (UCI Appliances), features
such as \textit{Press\_mm\_hg\_lag6} and \textit{RH\_1\_lag1} remain among the
top-ranked contributors in the zonal HXAI output despite noise-induced
attenuation, without disappearing or changing sign. By contrast, local-DP
explanations degrade rapidly as privacy increases: important features become
obscured, less relevant features are exaggerated, and rankings fluctuate,
leading to sign instability and loss of meaningful ordering. A similar pattern is observed in Fig.~\ref{fig:bars_all}(b) for UCI Household
Power, where \textit{Global\_intensity\_lag1} and
\textit{Sub\_metering\_3\_lag1} dominate the oracle explanation and retain their
relative ordering and direction under zonal HXAI. Local-DP explanations instead
distort this structure, inflating minor features and suppressing dominant ones,
which obscures the primary drivers of demand. Together, these results indicate
that zonal HXAI preserves semantic structure, feature identity, approximate
dominance, and directional consistency, rather than exact numerical fidelity,
yielding explanations that remain interpretable under strong privacy
constraints.

These observations demonstrate that HXAI preserves the semantic structure of
explanations, feature identity, relative ranking, and directional influence,
rather than precise numerical values. While quantitative metrics capture
average deviations (Section~\ref{sec:quantitative_results}), they do not by
themselves guarantee explainability. Local-DP may yield numerically bounded
but semantically unstable explanations, whereas zonal HXAI maintains stable
rankings and directions, enabling analysts to reliably identify what matters
most and why under strict privacy constraints. Figure~\ref{fig:rankflow_all} shows how the importance of the top ten features shifts as privacy noise increases for two HXAI variants: fixed-$\epsilon$ allocation (Fig.~\ref{fig:rankflow_all}(a)) and negotiated-$\epsilon$ allocation (Fig.~\ref{fig:rankflow_all}(b)). Despite increased privacy noise, the most important features remain in the top ranks, indicating that HXAI preserves the overall importance structure of explanations. In the fixed-$\epsilon$ setting, feature ranks show minor fluctuations, suggesting uniform privacy allocation is adequate. Conversely, the negotiated-$\epsilon$ setting exhibits tighter and more stable rank trajectories, especially for the Household Power dataset, indicating that adaptive privacy allocation better stabilizes explanations. Both panels show rank trajectories close to diagonal with few crossings, confirming that HXAI retains the relative importance of features under differential privacy while providing reliable explanations. Figure~\ref{fig:rankflow_compare_appliances}(a)-\ref{fig:rankflow_compare_appliances}(b) shows how the importance of the top ten features in the UCI Appliances dataset evolves with increasing privacy noise, comparing Central DP-SHAP and HXAI. Under Central DP-SHAP, feature rankings become unstable as noise increases, leading to frequent rank changes where dominant features may lose importance. This instability makes it difficult to identify key influences on model predictions. In contrast, HXAI maintains a stable rank structure across the privacy spectrum, with trajectories close to diagonal. This consistency, attributed to HXAI's hierarchical aggregation mechanism, highlights a limitation of centralized noise injection: while providing differential privacy, it disrupts the semantic integrity of explanations, resulting in unstable interpretations.
\begin{figure*}[htbp]
    \centering
    \begin{minipage}[htbp]{0.48\linewidth}
        \centering
        \includegraphics[width=\linewidth]{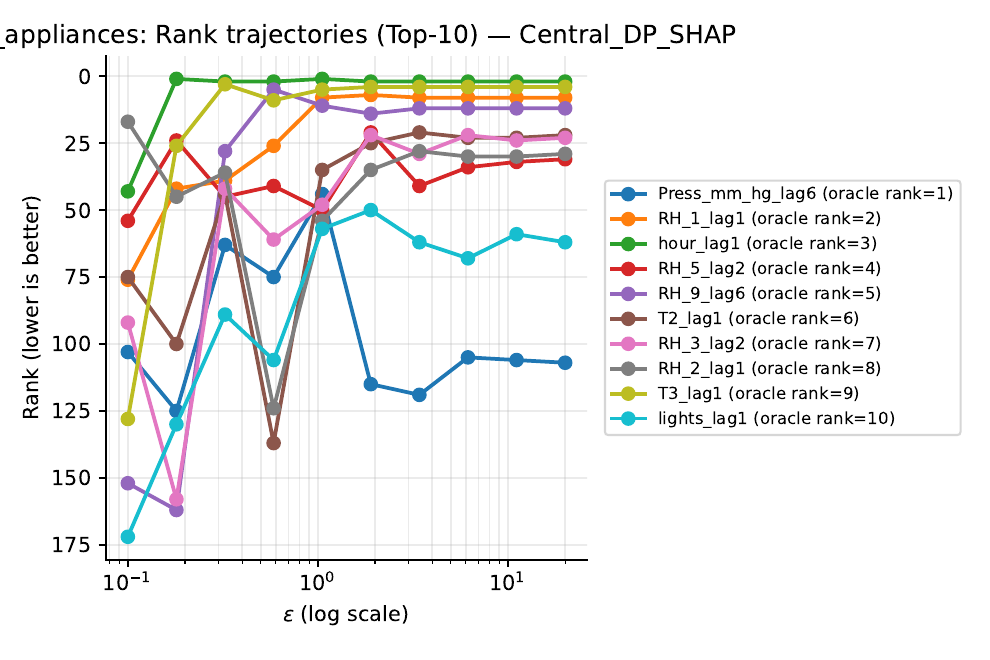}
        \\[-0.6em]
        (a) Central DP-SHAP
    \end{minipage}
    \hfill
    \begin{minipage}[htbp]{0.48\linewidth}
        \centering
        \includegraphics[width=\linewidth]{qual_uci_appliances_ranktraj_HXAI_negotiated_eps.pdf}
        \\[-0.6em]
        (b) HXAI (negotiated $\epsilon$)
    \end{minipage}

    \caption{Comparison of Top-10 feature rank-flow trajectories under
    increasing privacy noise for the UCI Appliances dataset. Central
    DP-SHAP exhibits unstable and crossing rank trajectories even in a
    homogeneous attribution regime, while HXAI preserves near-diagonal,
    stable rank structure.}
    \label{fig:rankflow_compare_appliances}
\end{figure*}
Across all datasets, HXAI consistently preserves semantic explanation structure under differential privacy, outperforming centralized DP-SHAP in ranking stability and directional consistency. Uniform privacy allocation is sufficient in homogeneous attribution regimes, while negotiated allocation provides clear benefits under heterogeneous feature importance. These results validate semantic stability, not numerical precision, as the key criterion for explainable explanations under privacy.

\section{System-Level Experimental Evaluation}
\label{sec:system_experiments}

In this section we shift focus to validating HXAI as a deployable framework for smart energy management. Rather than competing against baseline models using static datasets, we demonstrate how HXAI functions end-to-end in realistic settings involving households, zones, and energy providers, while satisfying differential privacy constraints. The experiments aim to answer questions unaddressed by offline datasets, such as the use of local explanations by households, the impact of zonal aggregation on utility under privacy constraints, the negotiation of privacy budgets, and how explainability leads to actionable interventions, such as pricing schemes and appliance upgrade incentives. The following system-level experiments demonstrate feasibility and end-to-end behavior under controlled conditions rather than field-deployed validation. In real DSO deployments, data volume, regulatory constraints, and heterogeneous household participation affect parameter calibration but not the hierarchical design of HXAI. HXAI does not require centralized model training, gradient exchange, or synchronized updates. Its computational cost scales linearly with the number of participating households and queries, making it suitable for large-scale zonal and city-level deployments.
 \begin{figure}[htbp]
    \centering
    \includegraphics[width=0.9\linewidth, height = 6 cm ]{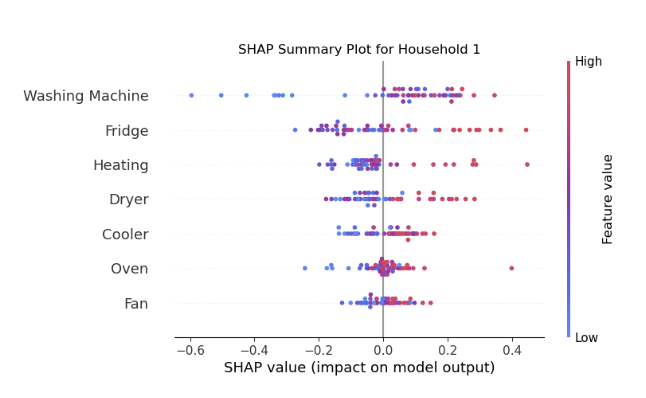}
    \caption{Noise-free local SHAP explanation for an individual household.
    The plot shows how different appliances contribute to energy
    consumption over time. This report is computed and consumed locally
    within the household’s Private Data Box and is never shared
    externally.}
    \label{shaploc}
\end{figure}
\subsection{Experimental Setup}
\label{sec:system_setup}
All experiments were conducted on a MacBook Air (Apple M1 chip) with 8 CPU cores and 8 GB RAM, running macOS 14.4.1 and Python 3.9 in a Jupyter Notebook, with privacy protections enabled. We generated energy consumption data for 50 households, each with seven IoT appliances: Heating, Fridge, Fan, Oven, Washing Machine, Dryer, and Cooler, observed at 100 time points for consistency. This design allows us to isolate system-level phenomena, like privacy budget negotiation and intervention feedback loops, that are difficult to measure in observational datasets, focusing on decision support and explainability within formal privacy guarantees.

\subsection{Local and Zonal Analysis}
\label{sec:local_zonal}
HXAI delivers stakeholder-specific explainability: households receive detailed, noise-free local insights, while providers access aggregated, privacy-preserving zonal explanations. We evaluate whether explanations remain actionable and meaningful at both levels while enforcing strict privacy guarantees.
\subsubsection{Local (End) Level Analysis}
\begin{figure}[htbp]
    \centering
    \includegraphics[width=\linewidth]{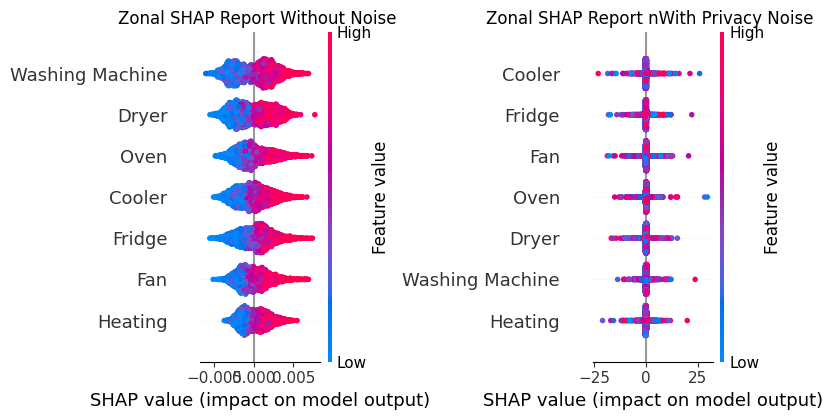}
    \caption{Zonal SHAP explanations before and after privacy-preserving
    noise application. While noise introduces small numerical
    variability, the relative importance and ordering of dominant
    appliances remain stable.}
    \label{figzonal}
\end{figure}
At the local level, HXAI provides households with full-resolution, noise-free
explanations computed entirely within the Private Data Box. These explanations
are never shared externally. While households already have access to raw
appliance-level consumption data $D_{\mathrm{local},i}$, such data only reports
how much energy each device consumes. In contrast, local SHAP explanations
quantify how individual appliances influence the model’s predictions, which is
a fundamentally different and model-dependent signal. In particular, SHAP
distinguishes appliances that actively drive predicted demand peaks or
price-relevant forecasts from those that merely consume energy without
significantly affecting predictions, information that cannot be reliably
inferred from raw consumption alone. 
\begin{figure}[htbp]
    \centering
    \includegraphics[width=0.9\linewidth]{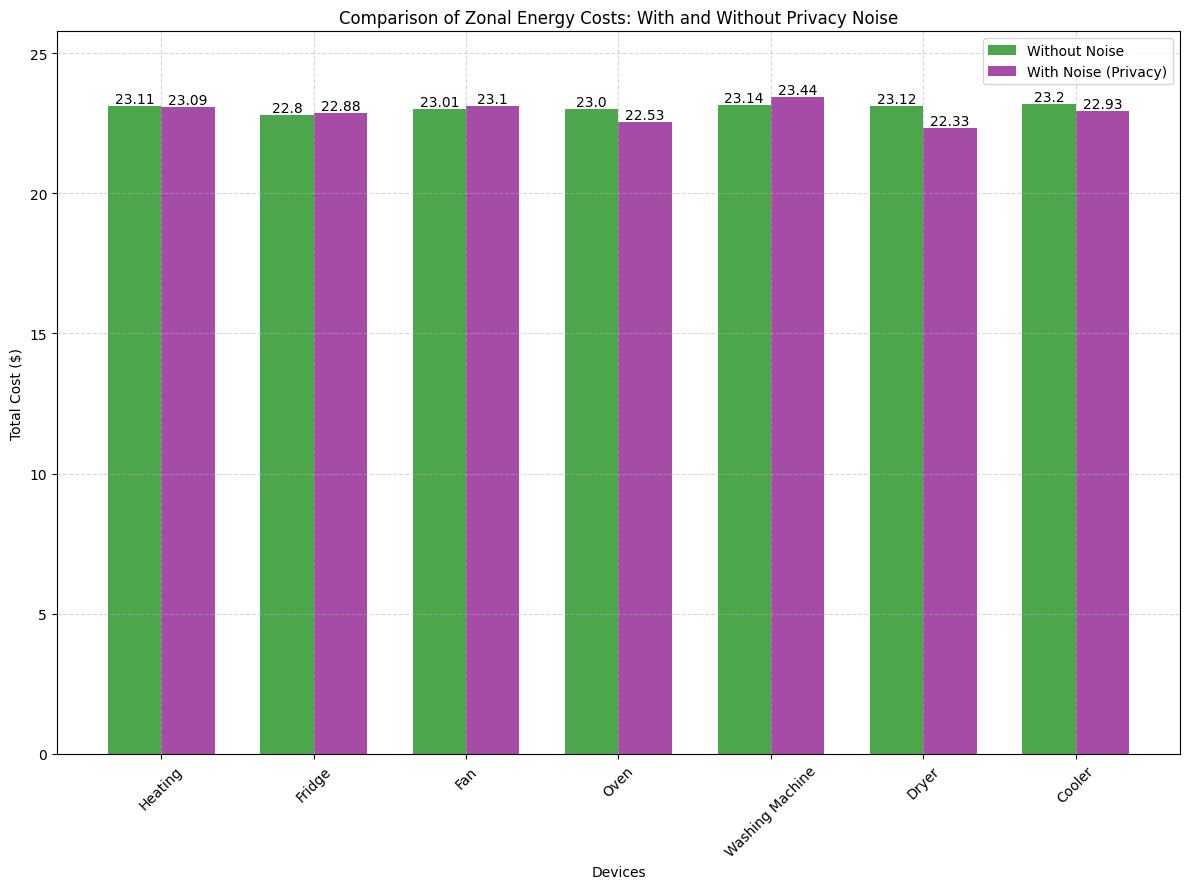}
    \caption{Zonal energy consumption report with and without privacy
    noise. Aggregation suppresses noise while preserving overall demand
    patterns relevant for grid operation.}
    \label{fig11}
\end{figure}
Figure~\ref{shaploc} presents a detailed local SHAP report for an individual
household, illustrating appliance-level explanatory contributions over time.
The washing machine exhibits the widest SHAP value distribution, indicating
that its influence on predicted demand varies strongly with usage timing,
cycle duration, and interaction with other appliances. Although heating may
consume more energy on average, its SHAP distribution is more balanced,
reflecting a comparatively stable and less peak-driving role in the
prediction. Appliances such as ovens and fans display smaller SHAP values,
indicating limited predictive relevance, even when their instantaneous energy
usage is non-negligible.
\begin{figure*}[htbp]
    \centering
    \includegraphics[width=\textwidth]{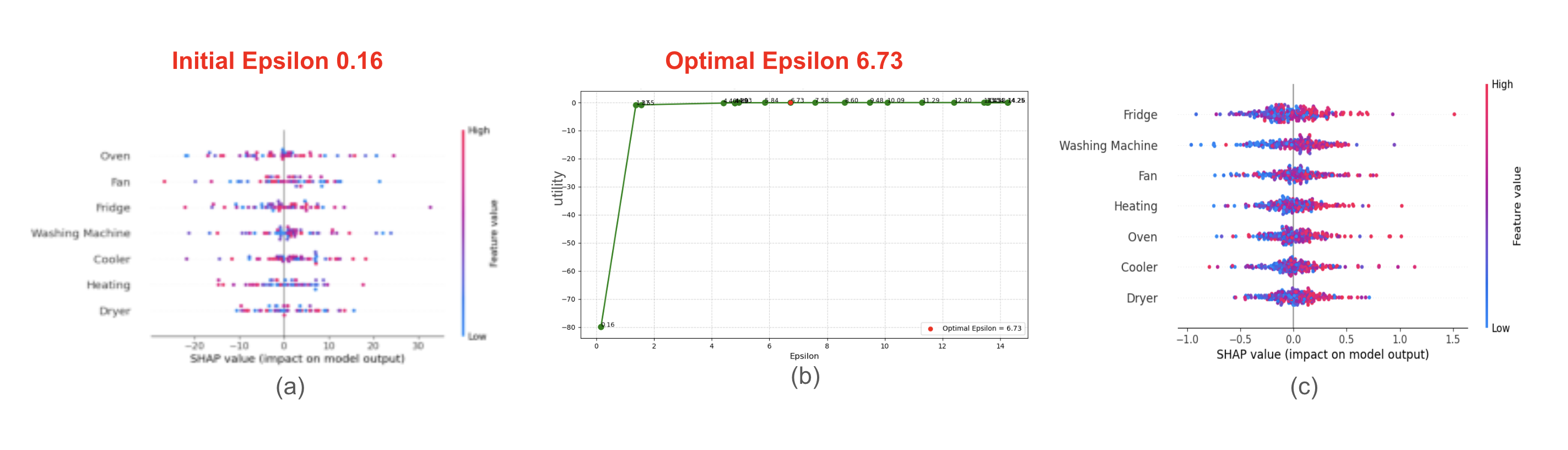}
    \caption{Privacy budget negotiation for an individual household:
    (a) initial strict privacy budget resulting in low utility,
    (b) negotiated privacy--utility trade-off with stabilized utility at
    $\epsilon=6.73$, and (c) corresponding local SHAP explanation at the
    negotiated budget, preserving semantic structure.}
    \label{fig4}
\end{figure*}
Crucially, this distinction cannot be inferred from raw consumption data alone. Raw measurements do not reveal whether an appliance causally drives predicted demand peaks, whether its effect is redundant with other devices, or how its influence varies with temporal context and feature interactions. Local SHAP explanations therefore move households from descriptive monitoring to actionable understanding. For instance, identifying the washing machine as a dominant explanatory driver enables off-peak rescheduling to reduce costs, even if it is not the largest energy consumer. More broadly, SHAP surfaces latent effects—such as efficiency, usage patterns, and program settings—that are invisible in raw data. Importantly, this local explainability provides the semantic foundation for zonal aggregation: only predictive influence, rather than raw usage, can be meaningfully aggregated under differential privacy, while detailed behavioral insights remain private.

\subsubsection{Zonal (Edge) Level Analysis}
HXAI aggregates local explanations at the zonal level and applies differential privacy to protect household data while enabling grid insights. Despite Laplace noise, overall consumption patterns and peak-demand signals remain consistent (Fig.~\ref{fig11}), and the relative importance of key appliances such as washing machines is preserved (Fig.~\ref{figzonal}). Although noise reduces precision, semantic stability is maintained, allowing energy providers to identify actionable patterns without compromising household confidentiality.
\begin{figure}[htbp]
    \centering
    \includegraphics[width=\linewidth ]{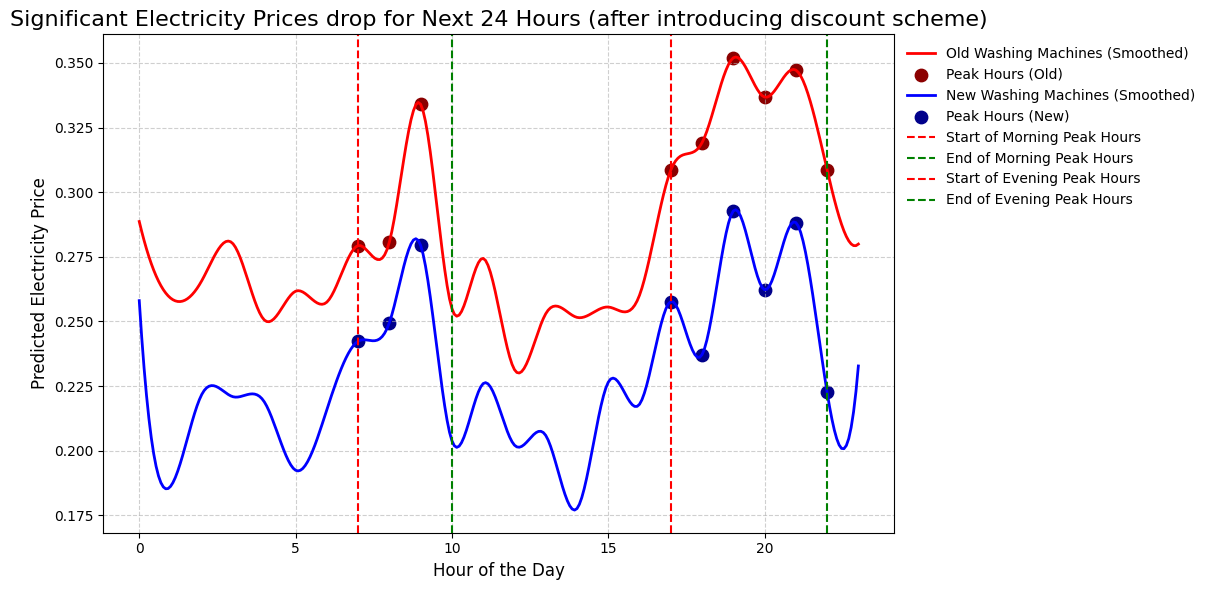}
    \caption{Predicted hourly electricity prices before and after an
    appliance upgrade intervention informed by zonal HXAI explanations.
    Peak prices are significantly reduced, resulting in a smoother and
    more stable price curve.}
    \label{fig8}
\end{figure}
\begin{figure*}[htbp]
    \centering
    \includegraphics[width=\linewidth]{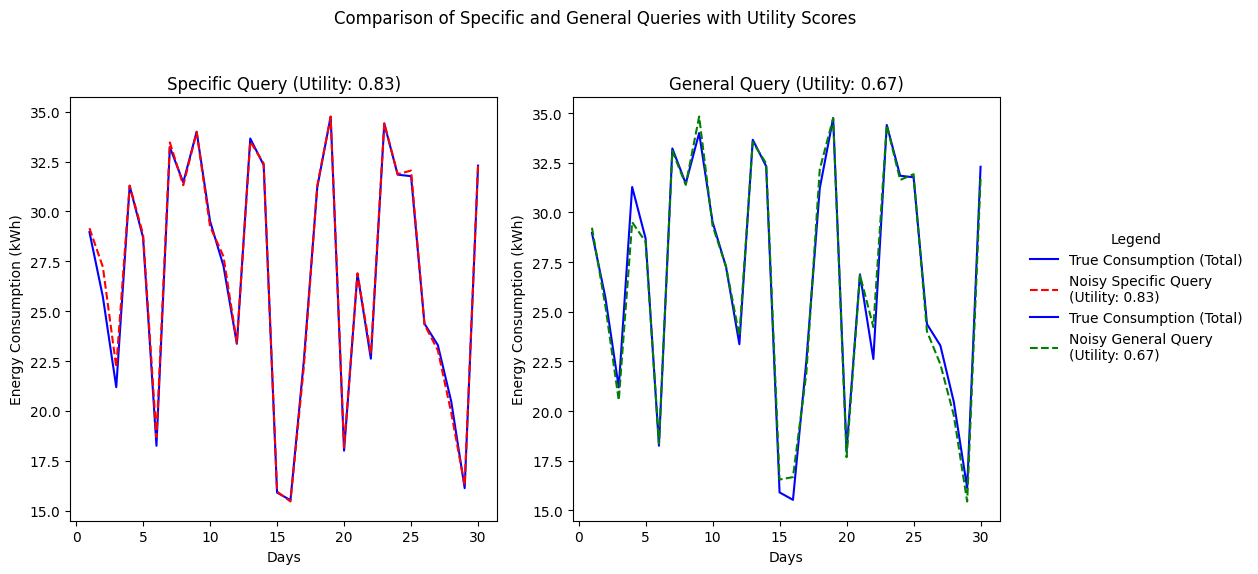}
    \caption{Utility comparison between specific (appliance-level) and
    general (aggregate) queries under differential privacy. Specific
    queries provide higher utility but consume more privacy budget.}
    \label{fig:specific_general_queries}
\end{figure*}

\subsection{Privacy-Utility Trade-off and Budget Negotiation}
\label{sec:negotiation}

A key aspect of privacy-preserving explainability is balancing household privacy with energy provider utility needs. Unlike static differential privacy mechanisms with fixed budgets, HXAI allows dynamic negotiation of privacy budgets at the explanation level.  Figure~\ref{fig4}(a) depicts the initial negotiation for a household, proposing a strict privacy budget ($\epsilon = 0.16$). While this maximizes privacy, it distorts explanations due to noise, diminishing their utility for tasks like demand forecasting and pricing analysis. This level of privacy may satisfy the household but provides limited value to the energy provider.

Figure~\ref{fig4}(b) shows the outcome of the negotiation process, revealing
$\epsilon = 6.73$ as an optimized point at which explanation utility stabilizes
while privacy remains acceptable to the household. This optimized point
represents a voluntary compromise: lower noise preserves the semantic structure
of explanations, while sufficient noise prevents leakage of sensitive
behavioral patterns. Importantly, acceptance of this privacy level is
household-driven and contingent on perceived benefit, such as improved
tariffs, discounts, or other economic incentives associated with higher-quality
explanations. Figure~\ref{fig4}(c) confirms that at this negotiated privacy budget, local SHAP
explanations retain their essential structure: dominant features remain
dominant, signs are consistent, and relative ordering is stable. Similar
negotiation dynamics and operating points are observed across households,
indicating that HXAI supports adaptive, incentive-aligned privacy management
without requiring manual tuning or dataset-specific heuristics. If a household
is unwilling to relax its privacy preference, the system defaults to stricter
privacy guarantees, resulting in noisier or more coarse-grained explanations
without penalizing the household.
\begin{figure}[htbp]
    \centering
    \includegraphics[width=\linewidth, height = 6 cm ]{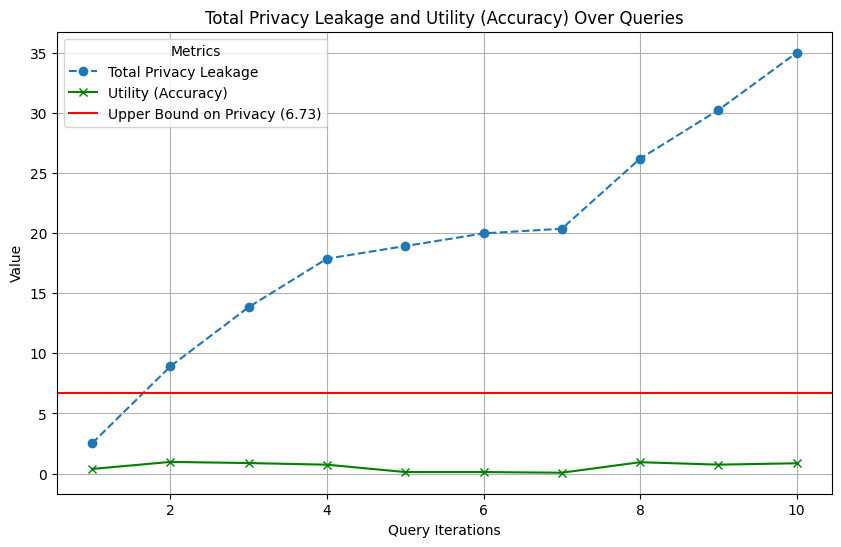}
    \caption{Accumulated privacy leakage and utility over repeated
    queries. Once the negotiated privacy budget is exhausted, households
    automatically deny further queries.}
    \label{fig:privacy_leakage}
\end{figure}
\subsection{Explainability-Driven Intervention and Grid Management}
\label{sec:intervention}
HXAI provides actionable grid insights while preserving household privacy. Zonal SHAP reports (Fig.~\ref{figzonal}) identify washing machines as major peak-demand contributors using privacy-preserving aggregation. We simulate a 20\% bill discount for households upgrading to energy-efficient models (20/50 participate). The 24-hour price forecast (Fig.~\ref{fig8}) shows lower peaks and a smoother curve compared to baseline, reflecting reduced aggregate demand. This demonstrates that HXAI enables explainability-driven interventions that stabilize load and reduce prices without exposing household-level data.
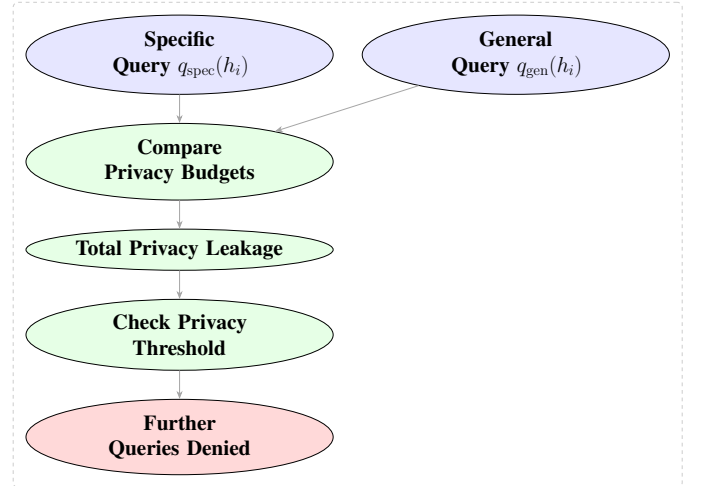
\begin{figure}[htbp]
\centering
\begin{adjustbox}{width=\linewidth, max height=8cm}
\begin{tikzpicture}[
    node distance = 8mm and 8mm,
    base/.style = {
        ellipse, draw,
        font=\LARGE\bfseries,
        text width=5.8cm,
        align=center,
        minimum height=1.6em,
        inner sep=4pt
    },
    query/.style   = {base, fill=blue!10},
    process/.style = {base, fill=green!10},
    output/.style  = {base, fill=red!15},
    arrow/.style   = {->, >=Stealth, thick, draw=gray!70}
]

\node[query] (specific) {Specific Query $q_{\mathrm{spec}}(h_i)$};
\node[query, right=of specific] (general) {General Query $q_{\mathrm{gen}}(h_i)$};

\node[process, below=of specific |- general.south] (compare)
{Compare Privacy Budgets};
\node[process, below=of compare] (total)
{Total Privacy Leakage};
\node[process, below=of total] (check)
{Check Privacy Threshold};
\node[output, below=of check] (deny)
{Further Queries Denied};

\draw[arrow] (specific) -- (compare);
\draw[arrow] (general) -- (compare);
\draw[arrow] (compare) -- (total);
\draw[arrow] (total) -- (check);
\draw[arrow] (check) -- (deny);

\node[
    fit=(specific)(general)(deny),
    inner sep=10pt,
    draw=gray!40,
    dashed,
    rounded corners
] {};

\end{tikzpicture}
\end{adjustbox}

\caption{Query handling and privacy accounting workflow in HXAI.}
\label{fig:privacy_query_workflow}
\end{figure}

\subsection{Query Mechanisms and Accumulated Privacy Leakage}
\label{sec:queries}

In operational settings, HXAI supports repeated explainability queries by differentiating query types and enforcing cumulative privacy-budget control through the Information Control Agreement (ICA).
\subsubsection{General and specific queries.}
HXAI supports two query types. \emph{General queries} provide aggregated, coarse-grained insights (e.g., zonal trends) with lower privacy cost. \emph{Specific queries} target fine-grained, appliance-level or temporal details and therefore consume a larger privacy budget. As shown in Fig.~\ref{fig:specific_general_queries}, specific queries yield higher explanation utility (83\%) than general queries (67\%), but at increased privacy cost. This trade-off follows the query-aware pricing mechanism in the ICA, where each query $q$ incurs a privacy cost $\varepsilon(q)$.

\subsubsection{Multiple queries and cumulative privacy leakage.}
To assess long-term privacy effects, we simulate repeated queries to a household’s Private Data Box under a fixed ICA. Each query consumes a negotiated privacy budget $\varepsilon_{i,j}^{\mathrm{adjust}}$, and cumulative leakage follows sequential composition. As shown in Fig.~\ref{fig:privacy_leakage}, leakage increases gradually until the household-specific threshold (e.g., $\varepsilon_i^{\max}=6.73$) is reached, after which further queries are denied. This prevents unbounded privacy erosion and enforces household preferences over time. By treating privacy as a consumable, trackable resource, HXAI enables sustainable long-term explainability under formal differential privacy guarantees.

\subsection{Discussion}
\label{sec:takeaway}

This work evaluates whether explanations remain meaningful under differential privacy. HXAI preserves feature identity, ranking, and directional influence across privacy regimes, whereas centralized DP-SHAP often causes rank instability and sign inversions. Unlike methods that depend on larger privacy budgets for numerical fidelity, HXAI maintains structural stability by design and supports local interpretability, privacy-preserving aggregation, and bounded long-term privacy leakage. Although evaluated on smart grids, HXAI generalizes to other privacy-sensitive domains where local interpretability and aggregate insights must coexist under formal privacy guarantees.

\textbf{Key takeaways:}
\begin{itemize}[leftmargin=*, itemsep=1pt, topsep=0pt]
    \item Explainability under DP requires semantic stability, not only numerical accuracy.
    \item Centralized DP explanations disrupt feature meaning and ordering.
    \item HXAI preserves stable, decision-relevant explanations via hierarchical aggregation.
    \item Adaptive privacy allocation and query-aware accounting enable practical, long-term deployment.
\end{itemize}

\subsection{Ethics and Reproducibility}
HXAI uses only simulated or publicly available datasets and introduces no new data collection or personally identifiable information. Household data are private by design, and differential privacy bounds leakage under repeated queries. The study poses no additional ethical risks beyond standard energy forecasting and XAI research. All datasets, code, and scripts are available in an anonymous repository: \url{https://anonymous.4open.science/r/HXAI-E3F2/README.md}.

\section{Conclusion}
In this paper, We introduce HXAI, a framework for explainable and privacy-preserving AI in distributed energy systems. By separating local explanation from data aggregation, HXAI protects household privacy while enabling intelligent grid management. Unlike traditional methods that degrade feature consistency under differential privacy, HXAI preserves feature identity and importance across datasets and privacy levels. The framework supports dynamic privacy budget negotiation and practical decision-making without exposing sensitive data. Overall, HXAI demonstrates that transparency and privacy can be achieved simultaneously in energy systems and other sensitive domains.

\bibliographystyle{IEEEtran}
\bibliography{main}

\end{document}